%% file: main.tex
\documentclass{article} 
\usepackage{nips15submit_e,times}

\usepackage{url}

\usepackage{times}
\usepackage{epsfig}
\usepackage{graphicx}
\usepackage{amsmath}
\usepackage{amssymb}
\usepackage{bbold}
\usepackage{dsfont}

\usepackage{enumerate}

\usepackage[toc,page]{appendix}
\usepackage[font=small,labelfont=bf]{caption}

\usepackage{xcolor}
\usepackage{amsthm}
\usepackage{amsfonts}
\usepackage{caption}
\usepackage{subcaption}
\usepackage{bm}
\usepackage{isomath}
\usepackage[numbers]{natbib}
\usepackage{footmisc}
\usepackage{stmaryrd}
\usepackage{fixltx2e}
\usepackage{dblfloatfix}
\usepackage{pbox}
\usepackage{capt-of}
\usepackage{multirow}

\newcommand{\xpt}{\edef\f@size{\@xpt}\rm}

\nipsfinaltrue

\def\ie{\emph{i.e.}}

\input{definitions.tex}

\title{\Large Domain Adaptation by Mixture of Alignments of Second- or Higher-Order Scatter Tensors}

\author{Piotr Koniusz\thanks{Both authors contributed equally.\newline .$\quad\!$This paper is accepted at CVPR'17.} \textsuperscript{ ,1,2}\qquad\qquad Yusuf Tas\textsuperscript{$*$,1,2}\qquad\qquad Fatih Porikli\textsuperscript{2}\\
$^1$Data61/CSIRO, $^2$Australian National University\\
firstname.lastname@\{data61.csiro.au, anu.edu.au\}
}

\newcommand\keywords[1]{}
\nipsfinalcopy 

\begin{document}

\maketitle

\input{abstract.tex}

\input{intro.tex}

\input{background.tex}
\input{approach.tex}
\input{experiments.tex}

\input{appendix.tex}


{\small

\input{egbib.bbl}
}

\end{document}

%% file: definitions.tex
\renewcommand\vec[1]{\ensuremath\boldsymbol{#1}}
\renewcommand\cdots{...}

\newcommand{\tY}{\vec{\mathcal{Y}}}

\newcommand{\vb}{\mathbf{b}}
\newcommand{\vy}{\mathbf{y}}

\newcommand{\tX}{\vec{\mathcal{X}}}

\newcommand{\mW}{\mathbf{W}}
\newcommand{\vx}{\mathbf{x}}

\newcommand{\mbr}[1]{\mathbb{R}^{#1}}

\newcommand{\idx}[1]{\mathcal{I}_{#1}}
\newcommand{\semipd}[1]{\mathcal{S}_{+}^{#1}}

\newcommand{\vzeta}{\boldsymbol{\zeta}}
\newcommand{\vzetabar}{\boldsymbol{\bar{\zeta}}}

\newcommand{\vphi}{\boldsymbol{\phi}}

\newcommand{\bigoh}{\mathcal{O}}

\newcommand{\vj}{\vec{j}}

\newcommand{\fnorm}[1]{\left\|{#1}\right\|_F}

\newcommand{\set}[1]{\left\{#1\right\}}

\DeclareMathOperator*{\argmin}{arg\,min}

\DeclareMathOperator*{\kronstack}{\uparrow\!\otimes}

\newcommand{\suptensorr}[2]{\mathfrak{S}^{#1}_{\times^{#2}}}

\newtheorem{definition}{Definition}

\newtheorem{proposition}{Proposition}

\def\eg{\emph{e.g.}}

\newcommand{\vOnes}{\mathbb{1}}

\newcommand{\overbar}[1]{\mkern 4mu\overline{\mkern-4mu#1\mkern-4mu}\mkern 4mu}
\newcommand{\overbartwo}[1]{\mkern 2mu\overline{\mkern-2mu#1\mkern-2mu}\mkern 2mu}

\newcommand{\mK}{\mathbf{K}}
\newcommand{\mKb}{\overbar{\mK}}
\newcommand{\mKbb}{\overbar{\mKb}}
\newcommand{\Kb}{\overbar{K}}
\newcommand{\Kbb}{\overbar{\Kb}}

\newcommand{\mKro}{{\mK^{q}}}
\newcommand{\mKbro}{{\mKb{\,\!}^{q}}}
\newcommand{\mKbbro}{{\mKbb^{q}}}

\newcommand{\cov}{\boldsymbol{\Sigma}}

\newcommand{\mPhi}{\boldsymbol{\Phi}}
\newcommand{\mLam}{\boldsymbol{\Lambda}}

\newcommand{\vmu}{\boldsymbol{\mu}}

\newcommand{\mP}{\boldsymbol{\Theta}}

\newcommand{\tXnb}{\mathcal{X}}
\newcommand{\tYnb}{\mathcal{Y}}

\newcommand{\dsAW}{$\mathcal{A}\!\!\shortrightarrow\!\mathcal{W}$}
\newcommand{\dsAD}{$\mathcal{A}\!\!\shortrightarrow\!\!\mathcal{D}$}
\newcommand{\dsWA}{$\mathcal{W}\!\!\shortrightarrow\!\!\mathcal{A}$}
\newcommand{\dsWD}{$\mathcal{W}\!\!\shortrightarrow\!\!\mathcal{D}$}
\newcommand{\dsDA}{$\mathcal{D}\!\!\shortrightarrow\!\!\mathcal{A}$}
\newcommand{\dsDW}{$\mathcal{D}\!\!\shortrightarrow\!\!\mathcal{W}$}

%% file: abstract.tex
\begin{abstract}
In this paper, we propose an approach to the domain adaptation, dubbed Second- or Higher-order Transfer of Knowledge ({So-HoT}), based on the mixture of alignments of second- or higher-order scatter statistics between the source and target domains. The human ability to learn from few labeled samples is a recurring motivation in the literature for domain adaptation. Towards this end, we investigate the supervised target scenario for which few labeled target training samples per category exist. 
Specifically, we utilize two CNN streams: the source and target networks fused at the classifier level. Features from the fully connected layers fc7 of each network are used to compute second- or even higher-order scatter tensors; one per network stream per class. 
As the source and target distributions are somewhat different despite being related, we align the scatters of the two network streams of the same class (within-class scatters) to a desired degree with our bespoke loss while maintaining good separation of the between-class scatters. 
We train the entire network in end-to-end fashion. We provide evaluations on the standard Office benchmark (visual domains), RGB-D combined with Caltech256 (depth-to-rgb transfer) and Pascal VOC2007 combined with the TU Berlin dataset (image-to-sketch transfer). 
We attain state-of-the-art results.
\end{abstract}


%% file: intro.tex
\section{Introduction}
\label{sec:intro}
Domain adaptation and transfer learning  are the problems widely studied in computer vision and machine learning communities~\cite{transfer_workshop_1995, transfer_workshop_2016}. They are directly inspired by the human cognitive abilities of generalizing to new concepts from very few data samples (cf. training from scratch on over a million of labeled images of the ImageNet dataset~\cite{ILSVRC15}). From psychological point of view, transfer of learning is ``{\em the dependency of human conduct, learning or performance on prior experience''}. This problem was introduced in 1901 under a notion of ``{\em transfer of particle}''~\cite{woodworth_particle}. In machine learning, 
transfer learning (or inductive learning) concerns {\em ``storing knowledge gained while solving one problem and applying it to a different but related problem''}~\cite{west_ml_transfer_def}. In practical computer vision and machine learning systems, transfer learning refers to {\em ``an ability of a system to recognize and apply knowledge and skills learned in previous tasks to novel tasks or new domains, which share some commonality''}. In general, given a new (target) task, the arising question is how to identify the commonality between this task and previous (source) tasks, and transfer knowledge from the previous tasks to the target one. Therefore, one has to address three questions: what to transfer, how to transfer, and when to transfer~\cite{tommasi_cvpr10}.

In what follows, we propose an approach to the domain adaptation, dubbed Second- or Higher-order Transfer of Knowledge ({\em So-HoT}), based on the mixture of alignments of second- and/or higher-order scatter statistics between the source and target domains. Specifically, we utilize second- and/or higher-order scatter tensors, one per each network stream per class, such that the first stream corresponds to the source domain while the second to the target. The scatters are built from the feature vectors produced by the {\em fc7} layer of AlexNet~\cite{krizhevsky_alexnet}. We propose that, as the source and target distributions are only partially related by their commonality, the scatters of the same class from both streams ({\em within-class scatters}) should be aligned to a desired degree to capture this commonality as an overlap between parts of the two distributions. At the same time, to achieve high classification accuracy, we maintain good separation between the scatters representing different classes ({\em between-class scatters}).
We devise a simple loss that brings each pair of within-class scatters closer in terms of their covariances as well as their corresponding means. Therefore, the CNN parameters stored by convolutional filters and weights of the target network regularized by the source data in this end-to-end fashion must produce statistics consistent with the source network. We view such a regularization paradigm as being motivated by the theory of  privileged learning~\cite{vapnik_privileged}. In our case, the statistics of the source network regularize the target (and vice-versa) whilst in the privileged learning, the side information regularizes the solution dictated by the empirical loss evaluated on the main data samples. See Figures \ref{fig:align_all} and \ref{fig:cnn_all} for illustrative examples.

\begin{figure*}[t]
\centering
%
%
\begin{subfigure}[b]{0.189\linewidth}
\centering\includegraphics[trim=0 0 0 0, clip=true, height=2.5cm]{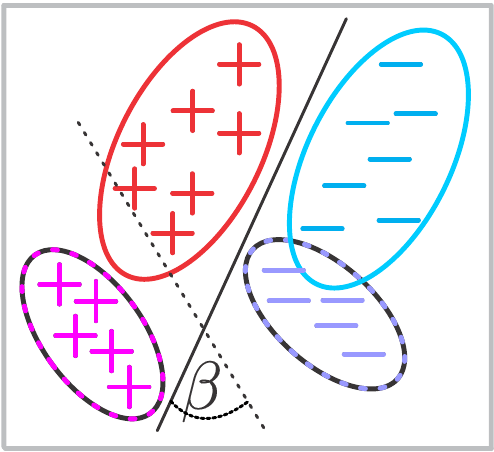}
\caption{\label{fig:align1}}
\end{subfigure}
\begin{subfigure}[b]{0.189\linewidth}
\centering\includegraphics[trim=0 0 0 0, clip=true, height=2.5cm]{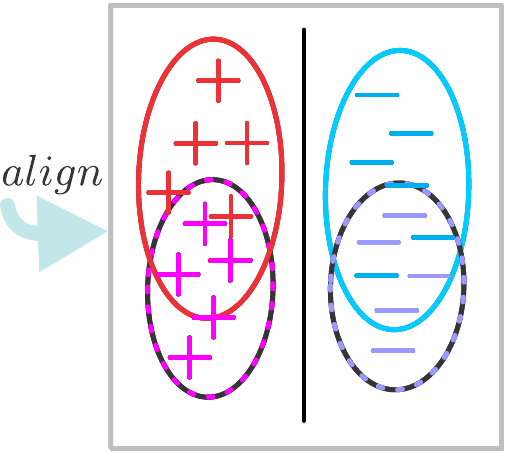}
\caption{\label{fig:align2}}
\end{subfigure}
\hspace{0.3cm}
\begin{subfigure}[b]{0.189\linewidth}
\centering\includegraphics[trim=0 0 0 0, clip=true, height=2.5cm]{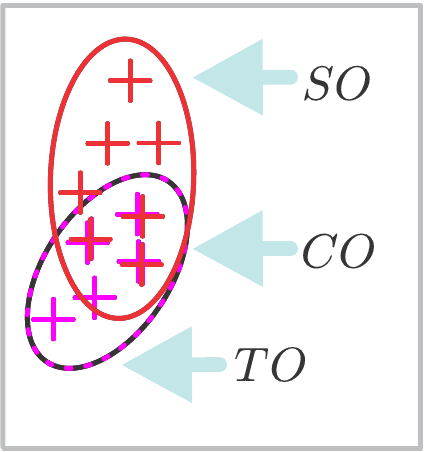}
\caption{\label{fig:align3}}
\end{subfigure}
\begin{subfigure}[b]{0.189\linewidth}
\centering\includegraphics[trim=0 0 0 0, clip=true, height=2.5cm]{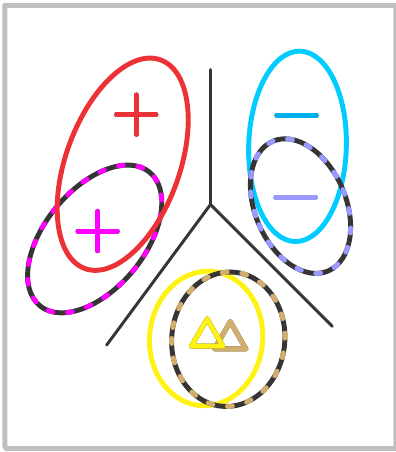}
\caption{\label{fig:align4}}
\end{subfigure}
\begin{subfigure}[b]{0.189\linewidth}
\centering\includegraphics[trim=0 0 0 0, clip=true, height=2.5cm]{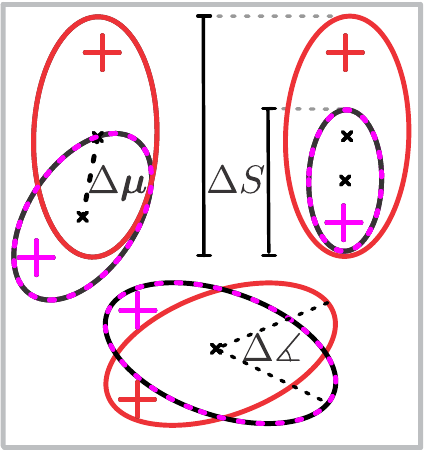}
\caption{\label{fig:align5}}
\end{subfigure}
\vspace{-0.3cm}
\caption{Our alignment problem. In Figure \ref{fig:align1}, a two-class toy problem with positive and negative samples ($+$) and ($-$) is given. The solid and dashed ellipses indicate the source and target domain distributions. The two hyperplane lines that separate ($+$) from ($-$) on the target data indicate large uncertainty (denoted as $\beta$) in the optimal orientation for the target problem. Figure \ref{fig:align2} shows that the source and target distributions can be aligned enough to separate well two classes for both the source and target problems. Figure \ref{fig:align3} shows that partially aligned distributions have the commonality ({\em CO}) as well as the source and target specific parts ({\em SO}) and ({\em TO}) that represent dissimilarity between the source and target. Figure \ref{fig:align4} depicts a multi-class problem. Beside of partially aligned means, the orientations of the source and target distributions are allowed to partially differ -- as a result, they \ie~fit better into the piece-wise linear decision boundary. Figure~\ref{fig:align5} shows that differences in means $\Delta\vmu$, scale/shear $\Delta S$ and orientation $\Delta\measuredangle$ of within-class scatters are all part of the alignment process.
}\vspace{-0.2cm}
\label{fig:align_all}
\end{figure*}

Furthermore, as distributions of the source and target domains may require different level of alignment per class (the commonality depends on the class label), we investigate not only an unweighted alignment loss (class-independent level of alignment) but also its weighted counterpart which learns one weight per class (class-specific levels of alignment).

Additionally, as we work with second- and/or higher-order tensors, we propose a kernelized variant of our alignment loss which provides computational speed-ups for typical domain adaptation datasets.

To summarize, our main contributions are: i) a novel loss that we call {\em So-HoT}, which defines the commonality between the source and target domains as the mixture of alignments of second- and/or higher-order scatter tensors, ii) unweighted and weighted variants of alignments, and iii) a fast kernelized alternative of our alignment loss.

Next, we detail the notion of domain adaptation and transfer learning, review the related literature and explain how our work differs from the state-of-the-art approaches.

\section{Related Work}
\label{sec:related}
Domain adaptation assumes that the transfer of knowledge takes place among two or more distinct domains \eg, e-commerce reviews and biomedical data. In contrast, transfer learning utilizes the same domain \eg, images of natural scenes with related but different distributions where the goal may be to learn objects of a new class while leveraging other already learned classes~\cite{thrun_nips, tommasi_cvpr10}. Not surprisingly, these both notions are often interchangeable \eg, natural images and sketches have related distributions but they come from distinct domains at the same time. Another example is a so-called domain shift \eg, bicycle in natural images vs. on-line retailer galleries. Transfer of knowledge may vary from simply carrying over discriminative information from a source to target domain under the same set of classes to inferring a solution to a new distinct task from a set of former ones~\cite{thrun_nips,intrator_diverse_tasks}. Domain adaptation comes in many flavors. Single- or multiple-source~\cite{crammer_multisource} setups are possible \eg, single stream of natural images vs. multiple streams supplied with photos of objects: on cluttered backgrounds, on a clear background, in a daytime or night setting, or even in multi-spectral setting. Moreover, the problem in hand may be homogeneous or heterogeneous~\cite{tommasi_cvpr10, yeh_cca_hetero} in nature \eg, identical source and target representations using RGB images vs. a source represented by a CNN trained on images~\cite{fukushima_nn,krizhevsky_alexnet} and a target using an 
LSTM~\cite{hopfield_net,rnn_early} which is trained on text corpora or video data \cite{Herath_Survey17}. The architecture in use may be shallow~\cite{frustrating_domain, frustrating_domain_return} or deep~\cite{ganin_jmlr_adversal} so that the commonality is established only at the classifier level or across entire source and target networks, respectively. Noteworthy is also recent trend in the CNN fine-tuning which by itself is a powerful domain adaptation and transfer learning tool~\cite{girshick_rich_feat,sermanet_overfeat} which requires large training datasets. Moreover, domain adaptation and transfer learning address problems such as: learning new categories from few annotated samples (supervised domain adaptation~\cite{chopra_icml_workshop, tzeng_transfer}), utilizing available unlabeled data (unsupervised~\cite{frustrating_domain_return, ganin_jmlr_adversal,Herath_ILS17} or semi-supervised domain adaptation~\cite{frustrating_domain, tzeng_transfer}), recognizing new categories in embedded spaces \eg, attribute-based, without any training samples (zero-shot learning~\cite{feifei_oneshot}).

In this paper, we investigate the case of a deep supervised single-source domain adaptation which can be easily extended to the multi-source and heterogeneous cases.

\vspace{0.05cm}
\noindent{\textbf{The Commonality.}}
Deep learning \cite{krizhevsky_alexnet,simonyan_vgg,Harandi_GB16} has been used in the context of domain adaptation in recent works \eg,~ \cite{tzeng_transfer, ganin_jmlr_adversal, chopra_icml_workshop, xiong_eccv16, kuzborskij_cvpr16, tommasi_eccv16,icml2015_long15}. These works differ in how they establish the so-called commonality between domains. In~\cite{tzeng_transfer}, the authors propose to align both domains via the cross entropy which ``maximally confuses'' both domains for supervised and semi-supervised settings. In~\cite{ganin_jmlr_adversal}, an unsupervised approach utilizes the assumption that predictions must be made based on features which cannot discriminate between the source and target domains. Specifically, they minimize a trade-off between the so-called source risk and the empirical 
divergence to find examples in the source domain indistinguishable from the target samples.

Our work differs from these methods in that we define the commonality as the desired degree of overlap between the second- and/or higher-order scatters of the source and target. After such an alignment, we let the non-overlapping tails of distributions also guide learning which results in a more general classifier (\ie~avoid the domain-specific bias). 

Moreover, in~\cite{chopra_icml_workshop}, the authors capture the ``interpolating path'' between the source and target domains using linear projections into a low-dimensional subspace which lies on the Grassman manifold. In~\cite{xiong_eccv16}, the authors propose to learn the transformation between the source and target by the deep model regression network. These two approaches assume that the source representation can be interpolated or regressed into the target as, given the nature of CNNs, they can approximate highly non-linear functions.

Our model differs in that our source and target network streams co-regularize each other to produce the commonality between the source and target distributions and accommodate the domain-specific parts that should not be aligned.

For visual domains, the commonality can be captured in the spatially-local sense. In~\cite{tommasi_eccv16}, the authors utilize so-called ``domainness maps'' which capture locally the degree of domain specificity. Similarly, in~\cite{kuzborskij_cvpr16}, the authors extract local patches of varying sizes at process each of these patches via CNNs. Our work is orthogonal to these techniques. We represent the commonality globally, however, our ideas could also be applied in a spatially-local setting. 

\vspace{0.05cm}
\noindent{\textbf{Correlation Methods.}}
Some recent works use correlation between the source and target distributions. Inspired by the Canonical Correlation Analysis (CCA), the authors of \cite{yeh_cca_hetero} utilize a correlation subspace as a joint representation for associating the data across different domains. They also use kernelized CCA. In~\cite{frustrating_domain_return}, the authors propose an unsupervised domain adaptation by the correlation alignment.

Our work is somewhat related in that we utilize second-order statistics. However, we align partially the class-specific source and target distributions to define the commonality (partial intersection of scatters) in the supervised setting. We also align partially the distribution means while the above unsupervised approaches use zero-centered feature vectors and the full alignment of the generic (c.f. class-specific) source and target distributions.
%
%
%
%
We detail how to learn the degree of alignment in an end-to-end fashion and introduce the kernelized loss between the second- and/or higher-order scatter tensors; all being novel propositions. 

\vspace{0.05cm}
\noindent{\textbf{Tensor Methods.}} Correlation approaches outlined above use second-order scatter matrices which are tensors of order $r\!=\!2$. In this work, we also investigate the applicability of higher-order scatters $r\!\geq\!3$ for alignment. Third-order tensors have been found useful for various vision tasks. For example, spatio-temporal third-order tensor on video data is proposed for action analysis in \cite{tensoraction2007}, non-negative tensor factorization is used for image denoising in~\cite{shashua2005non}, tensor textures are proposed for texture rendering in~\cite{vasilescu2004tensortextures}, and  higher order tensors are used for face recognition in~\cite{vasilescu2002multilinear}. A survey of multi-linear algebraic methods for tensor subspace learning is available in~\cite{lu2011survey}. The above applications use a single tensor, while our goal is to use tensors as the domain- and class-specific representations, similar to the sum-kernel approaches~\cite{me_tensor, sparse_tensor_cvpr, me_tensor_eccv16}, and apply them to alignment tasks.



%% file: background.tex
\begin{figure}[t]
\centering
%
\begin{subfigure}[b]{0.99\linewidth}
\centering\includegraphics[trim=0 0 0 0, clip=true, width=8.2cm]{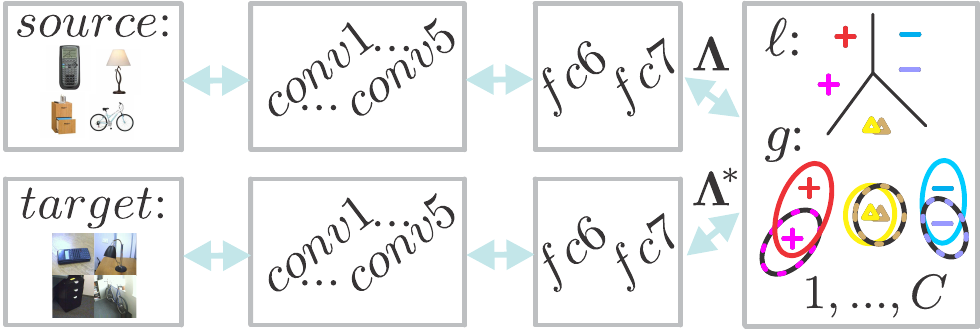}
\caption{\label{fig:cnn1}}
\vspace{-0.1cm}
\end{subfigure}
\begin{subfigure}[b]{0.99\linewidth}
\centering\includegraphics[trim=0 0 0 0, clip=true, width=8.2cm]{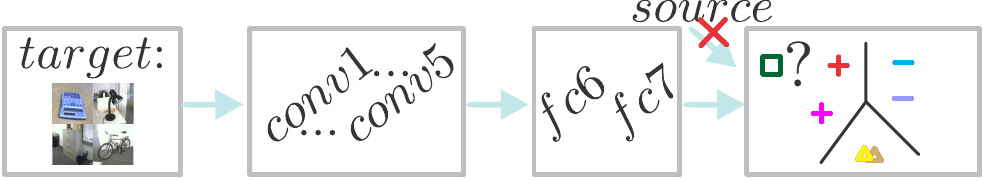}
\caption{\label{fig:cnn2}}
\end{subfigure}
\vspace{-0.2cm}
\caption{The pipeline. Figure \ref{fig:cnn1} shows the source and target network streams which merge at the classifier level. The classification and alignment loss $\ell$ and $g$ take the data  $\mLam$ and $\mLam^{*\!}$ from both streams and participate in end-to-end learning. At the test time, we use the target stream and the trained classifier as in Figure \ref{fig:cnn2}.
}\vspace{-0.3cm}
\label{fig:cnn_all}
\end{figure}

\section{Background}
\label{sec:preliminaries}
In this section, we review notations and the necessary background on scatter tensors, polynomial kernels and their linearizations, which are useful in deriving our mixture of alignments of second- and/or higher-order scatter tensors.

\subsection{Notations}
\label{sec:notations}
Let $\vx\in\mbr{d}$ be a $d$-dimensional feature vector. Then, we use $\tX\!=\!{\kronstack}_r\,\vx$ to denote the $r$-mode super-symmetric rank-one tensor $\tX$ generated by the $r$-th order outer-product of $\vx$, where the element of $\tX\!\in\!\suptensorr{d}{r}$ at the $\left(i_1,i_2,\cdots, i_{r}\right)$-th index is given by $\Pi_{j=1}^r x_{i_j}$. 
$\idx{N}$ stands for the index set $\set{1, 2,\cdots,N}$. We denote 
the space of super-symmetric tensors of dimension $d\!\!\times\!\!\cdots\!\!\times\!\!d$ with $r$ modes as $\suptensorr{d}{r}\!\!\subset\!\!\mbr{\times_r d}$, where $\mbr{\times_r d}$ is the space of tensors $\mbr{d\times\cdots\times d}$ with $r$ modes.
The Frobenius norm of tensor is given by  $\fnorm{\tX}\!\!=\!\!\!\sqrt{\sum\limits_{i_1,i_2,\cdots, i_{r}} \!\!\!\!\tXnb_{i_1,i_2,\cdots, i_{r}}^2}$, where $\tXnb_{i_1,i_2,\cdots, i_{r}}$ represents the $\left(i_1,i_2,\cdots, i_{r}\right)$-th element of $\tX$. Similarly, the inner-product between two tensors $\tX$ and $\tY$ is given by $\left\langle\tX,\tY\right\rangle\!=\!\!\!\!\!\!\sum\limits_{i_1,i_2,\cdots, i_{r}}\!\!\!\!\!\!\!\tXnb_{i_1,i_2,\cdots, i_{r}}\!\cdot\!\tYnb_{i_1,i_2,\cdots, i_{r}}$. Using Matlab style notation, the $(i_3,\cdots,i_r)$-th slice of $\tX$ is given by $\tX_{:,:,i_3,\cdots,i_r}$. The space of  positive semi-definite matrices is $\semipd{d}$. Lastly, $\vOnes$ denotes a vector with all coefficients equal one.

\subsection{Second- or Higher-order Scatter Tensors}
\label{sec:soho}
We define a scatter tensor of order $r$ as a mean-centered TOSST representation~\cite{sparse_tensor_cvpr}:
\begin{definition}
Suppose $\vphi_n\in\mbr{d},\forall n\in\idx{N}$ represent some data vectors, then a scatter tensor $\tX\in\suptensorr{d}{r}$ of order $r$ on these data vectors is given by:
\begin{equation}
\tX = \frac{1}{N}\sum_{n=1}^N{\kronstack}_r\,(\vphi_n-\vmu)\;\;\,\text{ and }\;\;\;\vmu=\frac{1}{N}\sum_{n=1}^N\vphi_n.
\label{eq:scatten}
\end{equation}
\end{definition}
In our supervised domain adaptation setting, the scatter tensors are obtained via applying~\eqref{eq:scatten} on the class-specific data vectors such as outputs of the {\em fc7} layer of AlexNet. When we need to highlight order $r$ of $\tX$, we write $\tX^{(r)}\!$.

The following properties of the scatter tensors are worth noting (see~\cite{sparse_tensor_cvpr} for proofs):
\begin{proposition}
For a scatter tensor $\tX\!\in\!\suptensorr{d}{r}$, we have:
\begin{enumerate}
\item Super-Symmetry: $\tX_{i_1,i_2,\cdots, i_{r}}\!=\!\tX_{\Pi(i_1,i_2,\cdots, i_{r})}$ for indexes $(i_1,i_2,\cdots, i_{r})$ and their any permutation $\Pi$. 
The number of unique coefficients of $\tX$ is $\binom{d+r-1}{r}$. 
\item Every slice is at least positive semi-definite for any even order $r\!\geq\!2$ and $\tX_{:,:,i_3,\cdots, i_{r}}\!\!\!\in\!\semipd{d},  \forall\!\,(i_3,\cdots,i_r)\!\in\!\idx{d}.\!$ 
For $r\!=\!2$, tensor $\tX$ also is a covariance matrix.
\item Indefiniteness for any odd order $r\!\geq\!1$, \ie, under a CP decomposition~\cite{lathauwer_hosvd}, it can have positive, negative, or zero entries in its core-tensor.
\end{enumerate}
\label{prop:tosst_prop}
\end{proposition}

Due to the indefiniteness of tensors of odd orders and potential rank deficiency, we restrict ourselves to work with the Euclidean distance between such scatter representations. Also, as the number of unique coefficients of $\tX$ is of order $\sim\!d^r$, which is prohibitive for $r\!\geq\!3$, we propose a light-weight kernelized variant of the Euclidean distance which avoids explicit use of tensors. The following easily verifiable two results will come handy in the sequel:
\begin{proposition}
Suppose we want to evaluate the Frobenius norm between tensors $\tX, \tX^*\!\!\in\!\suptensorr{d}{r}$, then it holds that: 
\begin{equation}
||\tX\!-\!\tX^{*}||_F^2=\left<\tX,\tX\right>-2\left<\tX,\tX^*\right>+\left<\tX^*\!\!,\tX^*\right>.
\label{eq:ten_inner1}
\end{equation}
\label{prop:ten_inner1}
\end{proposition}
\vspace{-1cm}
\begin{proof}
$\tX$ and $\tX^{*}$ can be vectorized and the Frobenius norm replaced by the $\ell_2$-norm for which the above expansion is known to hold. 
\end{proof}

\begin{proposition}
Suppose $\vx,\vy\in\mbr{d}$ are two arbitrary vectors, then for an ordinal $r\!>\!0$, we have:
\begin{equation}
\left<\vx,\vy\right>^{r} = \left\langle {\kronstack}_r\,\vx, {\kronstack}_r\,\vy\right\rangle.
\label{prop:bi_linearity}
\end{equation}
Moreover, for sets of vectors $\vx_n,\vy_{n'}\in\mbr{d}$, we have:
\begin{equation}
\sum_n\sum_{n'}\left<\vx_n,\vy_{n'}\right>^{r} = \Big\langle\sum_n{\kronstack}_r\,\vx_n, \sum_{n'}{\kronstack}_r\,\vy_{n'}\Big\rangle.
\label{eq:bi_linearity2}
\end{equation}
\label{prop:bi_linearity2}
\end{proposition}
\vspace{-0.7cm}
\begin{proof}
The expansion in \eqref{prop:bi_linearity} is derived in~\cite{me_tensor} while \eqref{eq:bi_linearity2} can be verified due to bilinear properties of the dot-product.
\end{proof}

%% file: approach.tex
\section{Proposed Approach}
\label{sec:approach}
In this section, we first formulate the problem of mixture of alignments of second- and/or higher-order scatter tensors, which precedes an exposition to our next two contributions: a weighted mixture of alignments and a kernelized approach which avoids explicit evaluations of scatters.

\subsection{Problem Formulation}
Suppose $\idx{N}$ and $\idx{N^*}\!$ are the indexes of $N$ source and $N^*\!$ target training data points. $\idx{N_c}$ and $\idx{N_c^*}\!$ are the class-specific indexes for $c\!\in\!\idx{C}$, where $C$ is the number of classes. Suppose we have feature vectors from {\em fc7} in the source network stream, one per image, and associated with them labels. Such pairs are given by $\mLam\!\equiv\!\{(\vphi_n, y_n)\}_{n\in\idx{N}}$, where $\vphi_n\!\in\!\mbr{d}$ and $y_n\!\in\!\idx{C}$, $\forall n\!\in\!\idx{N}$, as shown in Figure \ref{fig:cnn1}. For the target data, by analogy, we define pairs $\mLam^{*\!}\!\equiv\!\{(\vphi^*_n, y^*_n)\}_{n\in\idx{N}^*}$, where $\vphi^*\!\!\in\!\mbr{d}$ and $y^*_n\!\!\in\!\idx{C}$, $\forall n\!\in\!\idx{N}^*$. 
Class-specific sets of feature vectors are given as $\mPhi_c\!\equiv\!\{\vphi^c_n\}_{n\in\idx{N_c}}$ and $\mPhi_c^*\!\!\equiv\!\{\vphi^{c*}_n\}_{n\in\idx{N_c^*\!}}$, $\forall c\!\in\!\idx{C}$. Then, $\mPhi\!\equiv\!(\mPhi_1,\cdots,\mPhi_C)$ and $\mPhi^*\!\!\equiv\!(\mPhi^*_1,\cdots,\mPhi^*_C)$.
 Note that we use the asterisk symbol written in superscript (\eg~${\vphi}^*$) to denote variables associated with the target network whilst the source-related and generic variables have no such indicator. 
Below,  we formulate our problem as a trade-off between the classifier loss $\ell$ and the alignment loss $g$ which acts on the scatter tensors and related to them means:

\begin{align}
&\!\!\!\!\!\!\!\!\argmin\limits_{\;\;\substack{\mW,\vb,\mP,\mP^*\\\;\;\;\;\text{s. t. }||\vphi_n||_2^2\leq\tau,\\\;\;\;\;\;\;\;\;\,||\vphi^*_{n'}||_2^2\leq\tau,\\\;\;\;\;\forall n\in\idx{N}\!, n'\!\in\idx{N}^*}} 
\ell\!\left(\mW,\vb,\mLam\cup\mLam^*\!\right)+\lambda ||\mW||_F^2\label{eq:main_obj1}\\[-33pt]
&\qquad\qquad\quad\;\;+\!\underbrace{\frac{\sigma_1}{C}\!\!\sum_{c\in\idx{C}}\!||\tX_c\!-\!\tX_c^*||_F^2
+\!\frac{\sigma_2}{C}\!\!\sum_{c\in\idx{C}}\!||\vmu_c\!-\!\vmu_c^*||_2^2}_{g(\mPhi,\mPhi^*\!)}.\nonumber
\end{align}
For $\ell$, we use a generic loss used by CNNs, say Softmax. The matrix $\mW\!\in\!\mbr{d\times C}$ contains unnormalized probabilities (c.f. hyperplane of SVM), $\vb\!\in\!\mbr{C}$ is the bias term, and $\lambda$ is the regularization constant. Moreover, the union $\mLam\cup\mLam^*\!$ of the source and target training data reveals that we train one universal classifier for both domains\footnote{
\label{foot:classifier}
For VGG streams, we use a couple of domain-specific classifiers \eg,$\!$
$\ell\!\left(\mW,\vb,\mLam\right)\!+\!\ell\!\left(\mW^{*\!},\vb^{*\!},\mLam^{*\!}\right)\!+\!\lambda||\mW||_F^2\!+\!\lambda^{*\!}||\mW^{*\!}||_F^2\!+\!\beta'||\mW\!-\!\mW^{*\!}||_F^2$.
}.  
In Equation \eqref{eq:main_obj1}, separating the class-specific distributions is addressed by $\ell$ while bringing closer the within-class scatters of both network streams is handled by $g$ (as Figure \ref{fig:cnn_all} shows).
Specifically, our loss $g$ depends on two sets of variables $(\tX_1(\mPhi_1),\cdots,\tX_C(\mPhi_c)), (\vmu_1(\mPhi_1),\cdots,\vmu_C(\mPhi_C))$ and $(\tX^*_1(\mPhi^*_1),\cdots,\tX^*_C(\mPhi^*_C)), (\vmu^*_1(\mPhi^*_1),\cdots,\vmu^*_C(\mPhi^*_C))$ -- one set per network stream. Feature vectors $\mPhi(\mP)$ and $\mPhi^*\!(\mP^*\!)$ depend on the parameters of the source and target network streams $\mP$ and $\mP^*\!$ that we optimize over \eg, they represent coefficients of convolutional filters and weights of {\em fc} layers.
%
$\tX_c$, $\tX^*_c$, $\vmu_c$ and $\vmu^*_c$ denote the scatter tensors and means, respectively, one tensor/mean pair per network stream per class, evaluated as in \eqref{eq:scatten}. Lastly, $\sigma_1$ and $\sigma_2$ control the overall degree of the scatter and mean alignment, $\tau$ constraints the $\ell_2$-norm of feature vectors (needed if $\lambda$ is low). Derivatives of loss $g$ are given in Appendix \ref{app:der_sec_ord}.

In this work, we assume that highly non-linear CNN streams are able to rotate the within-class scatters sufficiently as dictated by our loss to yield a desired overlap of two scatters.  
Such an assumption is common in \ie~\cite{chopra_icml_workshop, xiong_eccv16}.

\subsection{Weighted Alignment Loss}

Below we propose a weighted variant of alignment loss $g$ that incorporates class-specific weights $\vzeta,\vzetabar\!\in\!\mbr{C}$ 
that adjust the degree of alignment per class between the within-class scatters as well as related to them means. As the statistical literature states that combination of moments $m\!\!=\!1,\cdots,\infty$ can capture any distribution, we combine $r'\!\!=\!2,\cdots,r$ orders ($\tX_c^{(1)}\!\!=\!\tX_c^{*(1)}\!\!=\!0$ due to data centering):
%
\begin{align}
& g^{(r)}(\mPhi,\mPhi^*\!\!\!,\{\vzeta_{r'\!}\}_{r'\!\in\idx{r}}\!,\vzetabar)\!=\!\!\frac{\sigma_1}{rC}\!\!\!\!\sum_{r'\!\in\idx{r}\!\setminus\!\{\!1\!\}}\!\!\!\sum_{\;\;c\in\idx{C}\!\!}\!\!\!\zeta_{cr'}||\tX_c^{(r')}\!-\!\tX_c^{*(r')}||_F^2\nonumber\\[-0.0cm]
&\qquad\quad\!+\!\!\frac{\sigma_2}{C}\!\!\sum_{c\in\idx{C}}\!\!\overbartwo{\zeta}_c||\vmu_c\!-\!\vmu_c^*||_2^2\!+\!\!\frac{\alpha_1}{r}\!\!\!\!\!\!\sum_{r'\!\in\idx{r}\!\setminus\!\{\!1\!\}}\!\!\!\!  \!||\vec{\zeta}_{r'}\!\!-\!\!\vOnes||_2^2\!+\!\alpha_2||\vec{\overbartwo{\zeta}}\!\!-\!\!\vOnes||_2^2,\!\label{eq:main_obj2}\\[-0.5cm]\nonumber
\end{align}
where $\alpha_1$ and $\alpha_2$ control the degree of weight deviation. To use the weighted alignment, we replace the corresponding loss in Eq. \eqref{eq:main_obj1} by the alignment loss $g$ defined in \eqref{eq:main_obj2}. Then, we additionally minimize \eqref{eq:main_obj1} over $\vzetabar$ and a set $\{\vzeta_{r'\!}\}_{r'\!\in\idx{r}\!\setminus\!\{\!1\!\}}$ that determines contributions of tensors of order $r'\!\!=\!2,\cdots,r$.

\subsection{Kernelized Alignment Loss}
\label{sec:kerapp}

Evaluating scatter tensors during the gradient descent is costly, even if using covariances ($r\!=\!2$), as the typical size feature vectors from {\em fc7} is $d\!=\!4096$. 
Below we propose an efficient kernelization of the Frobenius norm on tensors of arbitrary order $r$.

\begin{proposition}
The inner-product of scatter tensors $\tX^{(r)}\!\!$, $\tY^{(r)}\!\!\in\!\suptensorr{d}{r}$ of order $r$ from  Eq. \eqref{eq:scatten}, can be written implicitly as a sum of entries of a polynomial kernel ${\mKbb^r}\!\in\!\mbr{N\times N^*\!\!}$, where $\Kbb^r_{nn'}\!=\!\left<\vx_n\!-\!\vmu,\vy_{n'}\!\!-\!\vmu^*\!\right>^r$,  and $\vx_n\!\in\!\mbr{d}\!, \forall n\!\in\!\idx{N}$ and $\vy_{n'}\!\in\!\mbr{d}\!, \forall {n'}\!\in\!\idx{N^{*\!}}$ are some $N$ and $N^*\!$ feature vectors (that form $\tX^{(r)}\!$ and $\tY^{(r)}\!$), $\vmu$ and $\vmu^*\!$ are their means. Then:
\begin{align}
&\big\langle\tX^{(r)}\!\!,\tY^{(r)}\big\rangle\!=\!\frac{1}{NN^*\!}\!\sum_n\!\sum_{n'}\!\left<\vx_n\!\!-\!\vmu,\vy_{n'}\!-\!\vmu^*\right>^{r}\!\!=\!\frac{1}{NN^*\!}\vOnes^T{\mKbb^r}\vOnes.\nonumber\\[-12pt]
&
\label{eq:kernelised1}
\end{align}
\label{prop:kernelised1}
\end{proposition}
\vspace{-1cm}
\begin{proof}
Substituting $\vx_n\!-\!\vmu$ and $\vy_{n'}\!-\!\vmu^*\!$ into Proposition \ref{prop:bi_linearity2}, the proof follows.
\end{proof}

\begin{proposition}
Suppose we have polynomial kernels
${\mK^r}\!\in\!\mbr{N\times N\!}$, ${\mKb^r}\!\in\!\mbr{N^*\!\times N^*\!\!}$ and ${\mKbb^r}\!\in\!\mbr{N\times N^*\!\!}$ 
defined as $K^r_{nn'}\!=\!\left<\vx_n\!-\!\vmu,\vx_{n'}\!\!-\!\vmu\right>^r$,
$\Kb^r_{nn'}\!=\!\left<\vy_n\!-\!\vmu^*\!,\vy_{n'}\!\!-\!\vmu^*\!\right>^r$ and
$\Kbb^r_{nn'}\!=\!\left<\vx_n\!-\!\vmu,\vy_{n'}\!\!-\!\vmu^*\!\right>^r$, where $\vx_n, \vy_{n'}, \vmu, \vmu^*\!, N, N^*\!$ are defined as in Proposition \ref{prop:kernelised1}. The Frobenius norm between two scatter tensors $\tX^{(r)}\!\!,\tY^{(r)}\!\!\in\!\suptensorr{d}{r}$ of order $r$, which are defined in Eq. \eqref{eq:scatten}, can be expressed implicitly as:
\begin{align}
||\tX^{(r)}\!\!-\!\tX^{*(r)\!}||_F^2\!=\!\frac{1}{N^2}\vOnes^T \mK^r\vOnes\!+\!\frac{1}{{N^*}^2}\vOnes^T{\mKb^r}\vOnes
\!-\!\frac{2}{NN^*\!}\vOnes^T{\mKbb^r}\vOnes.
\label{prop:kernelised2}
\end{align}
%
%
\end{proposition}
%
\begin{proof}
Combining Proposition \ref{prop:ten_inner1} with \ref{prop:kernelised1}, the proof follows.
\end{proof}

Derivatives of \eqref{prop:kernelised2} are in Appendix \ref{app:der_ker_ord}. Equation \eqref{prop:kernelised2} can be evaluated on class-specific feature vectors and substituted directly into the loss functions in \eqref{eq:main_obj1} and \eqref{eq:main_obj2}. This way, we obtain two different regimes for evaluating the Frobenius norm on the scatter tensors: one explicit and one kernelized; both exhibiting different strengths as detailed below.

\vspace{0.05cm}
\noindent{\textbf{Complexity.}} The Frobenius norm on the scatter tensors has complexity $\bigoh((N\!+\!N^*\!+\!1)D)$, where $D\!=\!\binom{d+r-1}{r}$ as detailed in Proposition \ref{prop:tosst_prop}. The kernelized variant proposed above has complexity $\bigoh((N^2+NN^*\!+N^*\!N^*\!)(d\!+\!\rho)$, where $\rho\!\leq\!\log{r}$ estimates the complexity of ``rising $x$ to the power of $r$''. As $\rho\!\ll\!d$, its cost is negligible and can be safely left out from the above analysis.

It is easy to verify that, for the standard domain adaptation problems with $N\!=\!20$ source and $N^{*}\!\!=\!3$ target training points per class, 
$d\!=\!4096$ and $r\!=\!2$, explicit evaluations of the Frobenius norm are $\sim\!52\!\times$ slower than the proposed by us kernelized substitute. 
For the same scenario but with the scatter tensor of order $r\!=\!3$, explicit evaluations of the Frobenius norm are not tractable, as they take $\sim\!143000\!\times$ more time than the kernelized substitute, which demonstrates the clear benefit of our approach. The kernelization makes Eq. \eqref{eq:main_obj2} tractable for $r\!>\!2$.

%% file: experiments.tex
\begin{figure}[t]
\centering
\begin{subfigure}[b]{0.13\linewidth}
\centering\includegraphics[trim=0 0 0 0, clip=true, height=1.5cm]{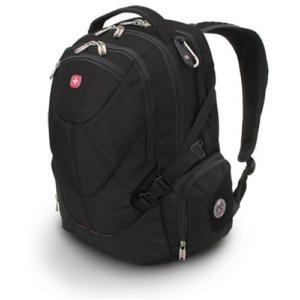}
\end{subfigure}
\begin{subfigure}[b]{0.13\linewidth}
\centering\includegraphics[trim=0 0 0 0, clip=true, height=1.5cm]{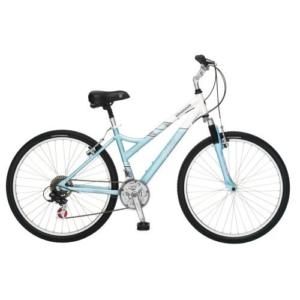}
\end{subfigure}
\begin{subfigure}[b]{0.13\linewidth}
\centering\includegraphics[trim=0 0 0 0, clip=true, height=1.5cm]{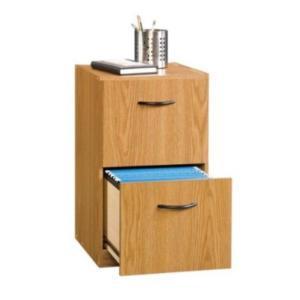}
\end{subfigure}
\begin{subfigure}[b]{0.13\linewidth}
\centering\includegraphics[trim=0 0 0 0, clip=true, height=1.5cm]{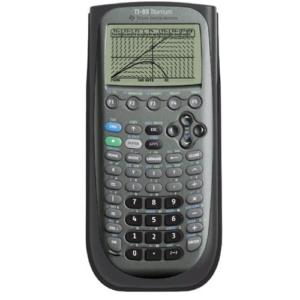}
\end{subfigure}
\begin{subfigure}[b]{0.13\linewidth}
\centering\includegraphics[trim=0 0 0 0, clip=true, height=1.5cm]{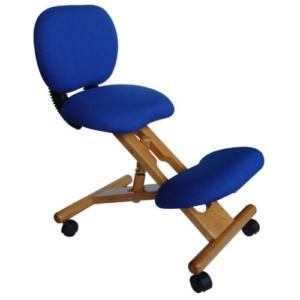}
\end{subfigure}\\
\vspace{0.1cm}
\begin{subfigure}[b]{0.13\linewidth}
\centering\includegraphics[trim=0 0 0 0, clip=true, height=1.5cm]{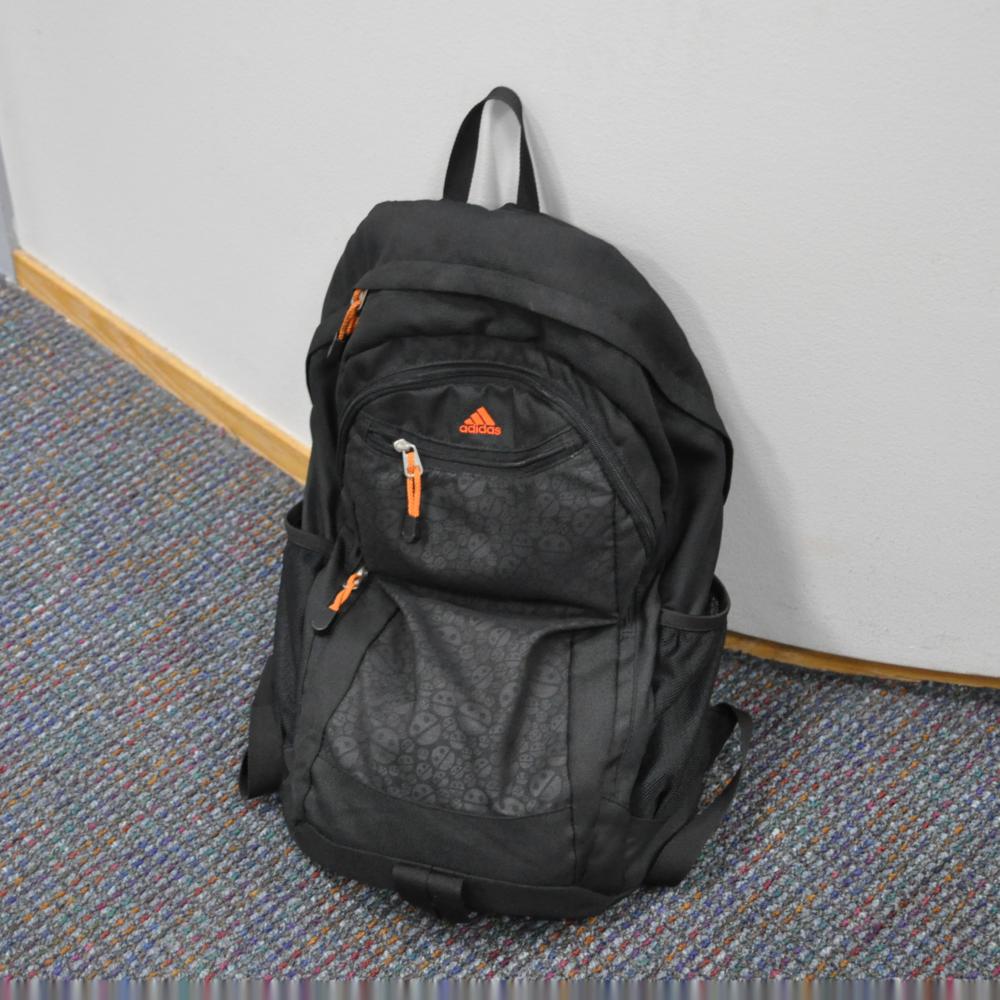}
\end{subfigure}
\begin{subfigure}[b]{0.13\linewidth}
\centering\includegraphics[trim=0 0 0 0, clip=true, height=1.5cm]{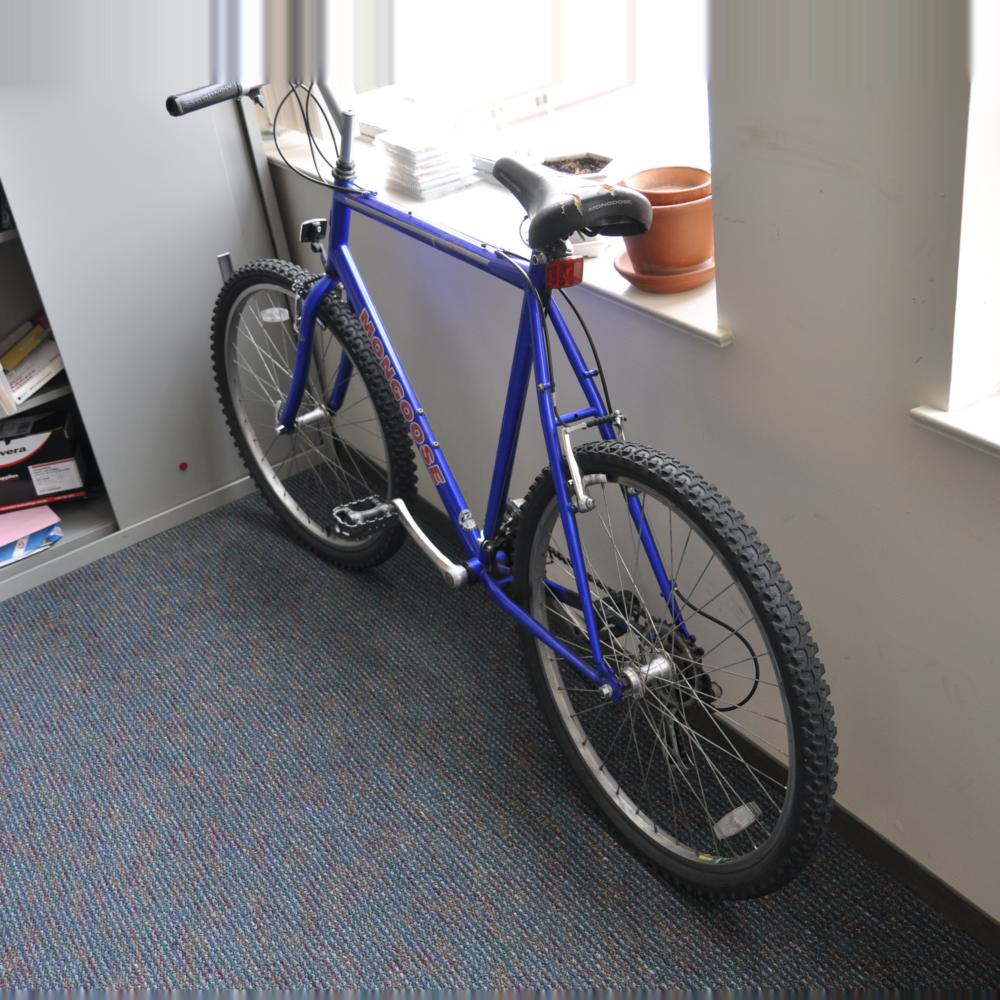}
\end{subfigure}
\begin{subfigure}[b]{0.13\linewidth}
\centering\includegraphics[trim=0 0 0 0, clip=true, height=1.5cm]{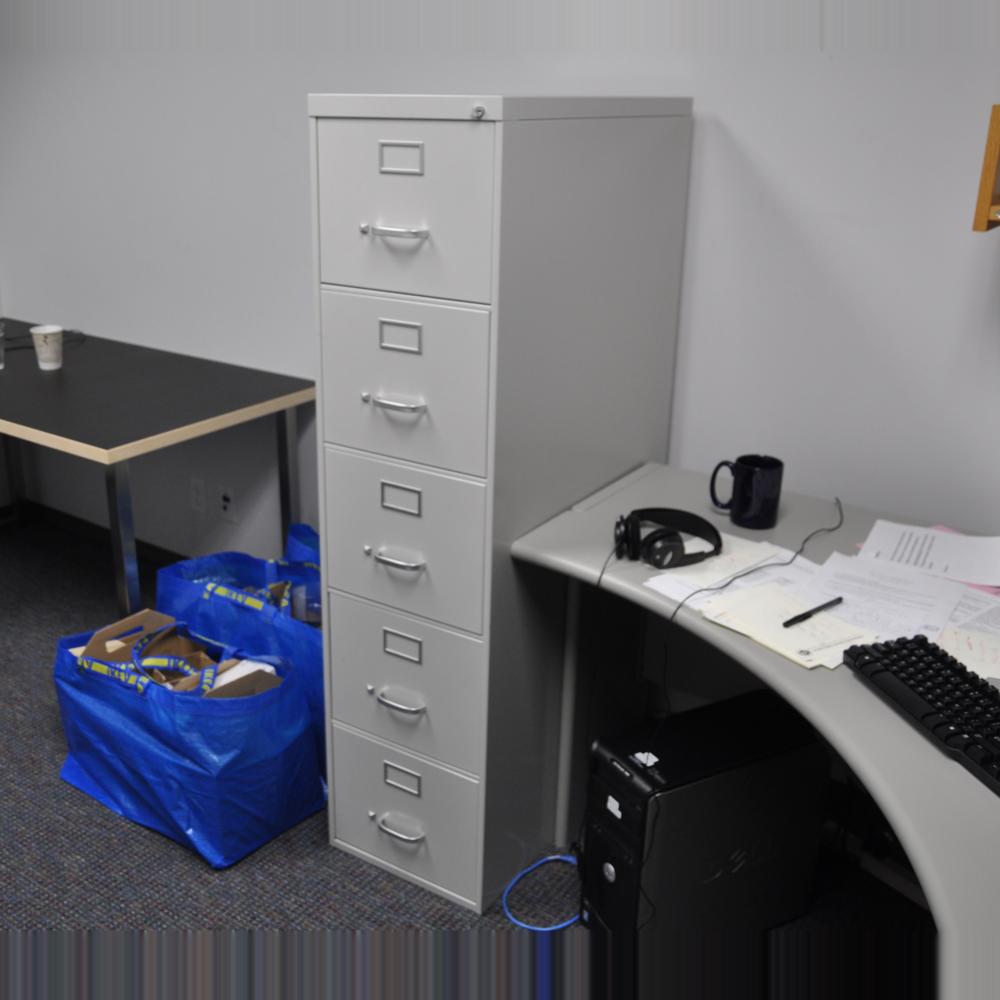}
\end{subfigure}
\begin{subfigure}[b]{0.13\linewidth}
\centering\includegraphics[trim=0 0 0 0, clip=true, height=1.5cm]{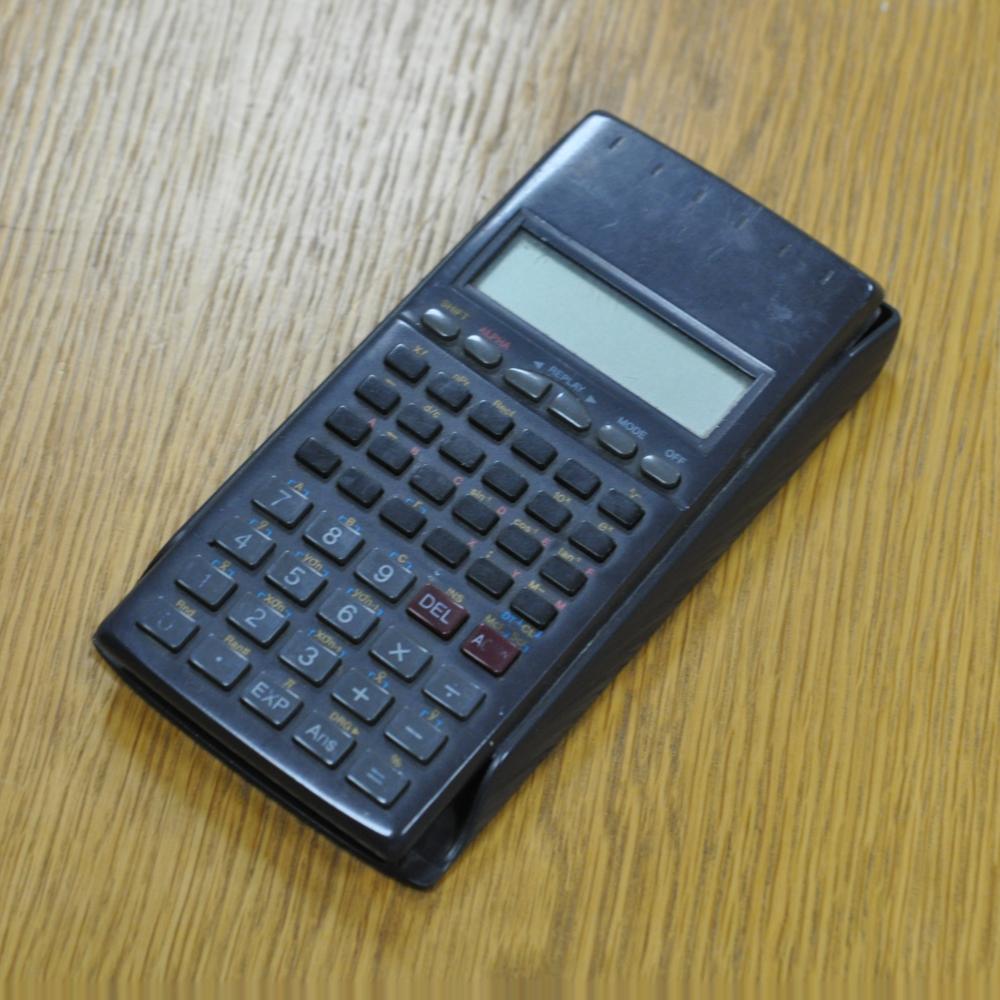}
\end{subfigure}
\begin{subfigure}[b]{0.13\linewidth}
\centering\includegraphics[trim=0 0 0 0, clip=true, height=1.5cm]{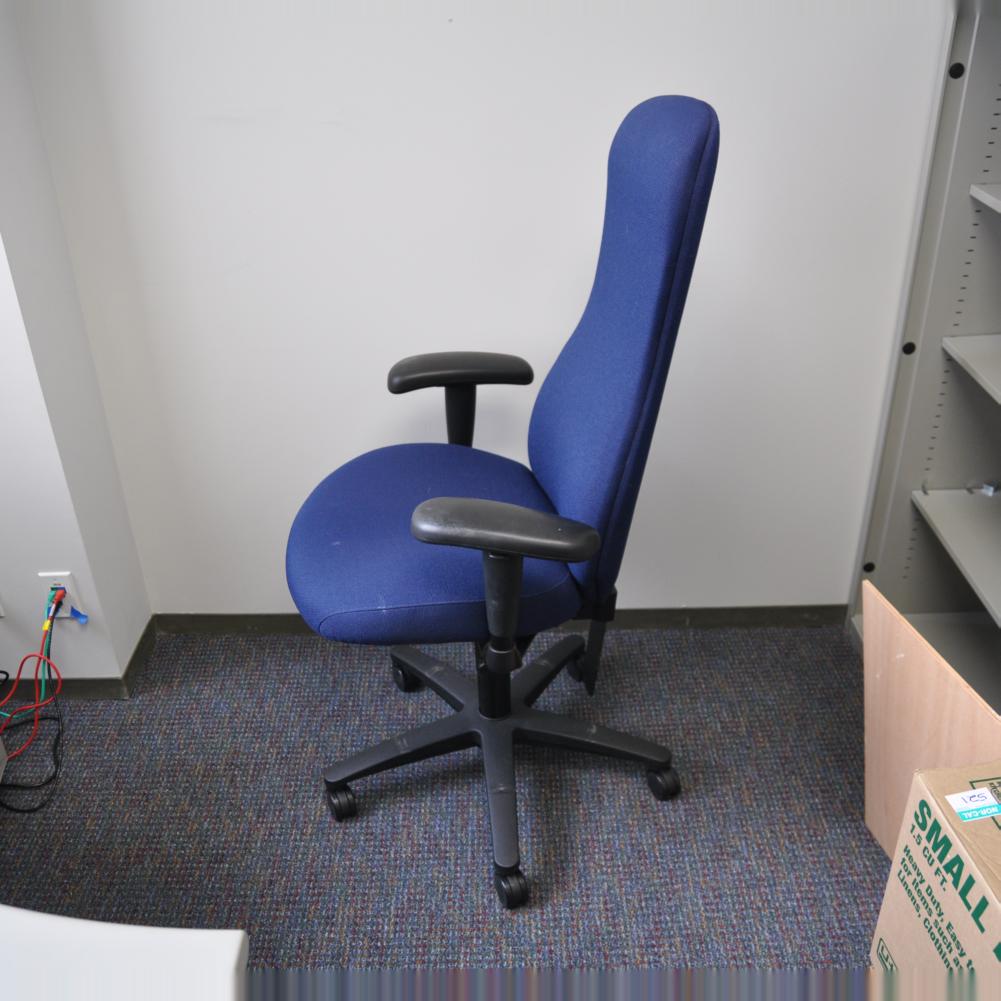}
\end{subfigure}\\
\vspace{0.1cm}
\begin{subfigure}[b]{0.13\linewidth}
\centering\includegraphics[trim=0 0 0 0, clip=true, height=1.5cm]{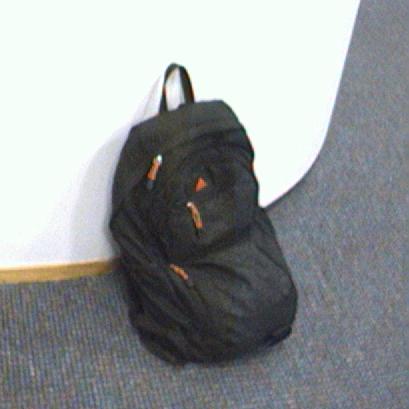}
\end{subfigure}
\begin{subfigure}[b]{0.13\linewidth}
\centering\includegraphics[trim=0 0 0 0, clip=true, height=1.5cm]{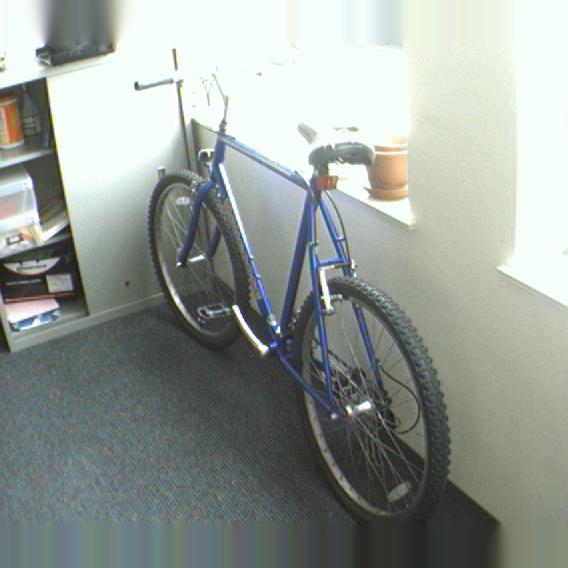}
\end{subfigure}
\begin{subfigure}[b]{0.13\linewidth}
\centering\includegraphics[trim=0 0 0 0, clip=true, height=1.5cm]{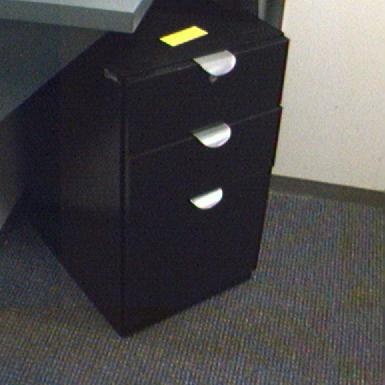}
\end{subfigure}
\begin{subfigure}[b]{0.13\linewidth}
\centering\includegraphics[trim=0 0 0 0, clip=true, height=1.5cm]{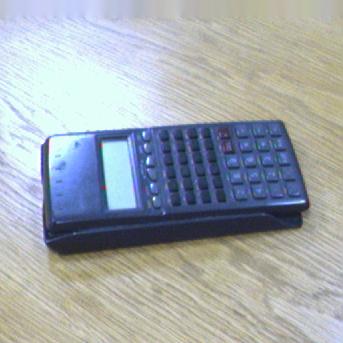}
\end{subfigure}
\begin{subfigure}[b]{0.13\linewidth}
\centering\includegraphics[trim=0 0 0 0, clip=true, height=1.5cm]{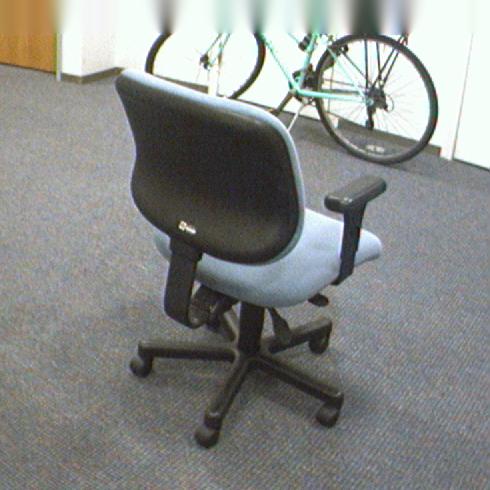}
\end{subfigure}
%
%
%
\caption{The Office dataset. Top, middle and bottom rows show examples from the Amazon, DSLR, and Webcam domains.  
}\vspace{-0.5cm}
\label{fig:office}
\end{figure}

\section{Experiments}

In this section, we present experiments demonstrating the usefulness of our framework. We start by describing datasets we use in evaluations.

\begin{table*}[!b]
\vspace{-0.2cm}
\begin{center}
{
\setlength{\tabcolsep}{0.25em}
\centering
\begin{tabular}{c c c|c c c c c c|c}
																																												&	&	& \dsAW  				& \dsAD   			& \dsWA   			& \dsWD						& \dsDA					& \dsDW 				& acc.		\\ \hline
\kern-0.6em&\kern-0.6em DLID & \cite{chopra_icml_workshop}																													& 51.9					& -							& -							& 89.9						& -							& 78.2					& 73.33\\
\kern-0.6em&\kern-0.6em DeCAF\textsubscript{6} S+T & \cite{donahue_decaf}																						& 80.7$\pm$2.3	& -							& -							& -								& -							& 94.8$\pm$1.2	& 87.75\\
\kern-0.6em&\kern-0.6em DaNN & \cite{dann_com}																																			& 53.6$\pm$0.2	& -							& -							& 83.5$\pm$0.0		& -							& 71.2$\pm$0.0	& 69.43\\
\kern-0.6em&\kern-0.6em Source CNN & \kern-0.0em\cite{tzeng_transfer}\kern-0.0em																		& 56.5$\pm$0.3	& 64.6$\pm$0.4	& 42.7$\pm$0.1	& 93.6$\pm$0.2 		& 47.6$\pm$0.1	& 92.4$\pm$0.3	& 66.23\\
\kern-0.6em&\kern-0.6em Target CNN & \kern-0.0em\cite{tzeng_transfer}\kern-0.0em																		& 80.5$\pm$0.5	& 81.8$\pm$1.0	& 59.9$\pm$0.3	& 81.8$\pm$1.0		& 59.9$\pm$0.3	& 80.5$\pm$0.5	& 74.06\\
\kern-0.6em&\kern-0.6em Source+Target CNN\kern-0.6em  & \kern+0.0em\cite{tzeng_transfer}\kern-0.0em 		& 82.5$\pm$0.9	& 85.2$\pm$1.1	& 65.2$\pm$0.7	& 96.3$\pm$0.5		& 65.8$\pm$0.5	& 93.9$\pm$0.5	& 81.48\\
\kern-0.6em&\kern-0.6em Dom. Conf.+Soft Labs.\kern-0.6em & \kern+0.0em\cite{tzeng_transfer}\kern-0.0em	& 82.7$\pm$0.8	& 86.1$\pm$1.2	& 65.0$\pm$0.5	& \textbf{97.6}$\pm$\textbf{0.2}		& 66.2$\pm$0.3	& \textbf{95.7}$\pm$\textbf{0.5}	& 82.22\\\hline
 &\kern-0.6em Source+Target CNN ({\em S+T})$\!\!\!\!\!\!\!\!\!\!\!\!$ & $\!\!\!\!\!\!\!\!$ & 82.4$\pm$2.0	& 85.5$\pm$0.9	& 65.1$\pm$1.4	& 95.8$\pm$0.8		& 66.0$\pm$1.2	& 94.3$\pm$0.6	& 81.53\\
\kern-0.6em&\kern-0.6em Second-order ({\em So})																																		& & \textbf{84.5}$\pm$\textbf{1.7} & \textbf{86.3}$\pm$\textbf{0.8}	&	\textbf{65.7}$\pm$\textbf{1.7}	&	97.5$\pm$0.7		& \textbf{66.5}$\pm$\textbf{1.0}								&	95.5$\pm$0.6	&		\textbf{82.68}\\
\end{tabular}
}
\end{center}
\vspace{-0.3cm}
\caption{Comparison of our second-order alignment loss ({\em So}) to the state of the art on the Office dataset.}
\label{tab:office}
\end{table*}

\subsection{Datasets}

\vspace{0.05cm}
\noindent{\textbf{Office dataset.}}  A popular dataset for evaluating algorithms against the effect of domain shift is the Office dataset~\cite{saenko_office} which contains 31 object categories in three domains: Amazon, DSLR and Webcam. The 31 categories
in the dataset consist of objects commonly encountered in office settings, such as keyboards, file cabinets, and laptops. The Amazon domain contains on average 90 images per class and 2817 images in total. As these images were captured from a website of online merchants, they are captured against clean background and at a unified scale. The DSLR domain contains 498 low-noise high resolution images ($4288\!\times\!2848$). There are 5 objects per category. Each object was captured from different viewpoints on average 3 times. For Webcam, the 795 images of low resolution ($640\!\times\!480$) exhibit significant noise and color as well as white balance artifacts. Otherwise, 5 objects per category were also used in the capturing process. Figure \ref{fig:office} illustrates the three domains. We distinguish the following six domain shifts: Amazon-Webcam (\dsAW), Amazon-DSLR (\dsAD), Webcam-Amazon (\dsWA), Webcam-DSLR (\dsWD), DSLR-Amazon (\dsDA) and DSLR-Webcam (\dsDW).

We evaluate across 10 randomly chosen data splits per domain shift. We follow the standard protocol for this dataset and, for each training source split, we sample 20 images per category for the Amazon domain and 8 examples per category for the DSLR and Webcam domains. From the training target splits, we sample 3 images per class per split per domain. We present results for the supervised setting and report accuracies on the remaining target images, as the the standard protocol for this dataset suggests. 

\vspace{0.05cm}
\noindent{\textbf{RGB-D-Caltech256 dataset.}} The RGB-D~\cite{lai_rgbd} and Caltech256~\cite{griffin_caltech256} datasets have been used as the source and target for evaluations of unsupervised domain adaptation problems~\cite{cvpr_rgbd-caltech,eccv_infbottleneck}. We use the 10 classes that are common between the two datasets \eg, calculator, cereal box, coffee mug, ball, tomato. We use 50 and 5/10 images per class in the source and target domains 
for the supervised setting. We test on the remaining target samples. We report the mean average accuracy over 5 data splits, that is, we select randomly the source and target data samples for each split.

\vspace{0.05cm}
\noindent{\textbf{Pascal VOC2007-TU Berlin dataset.}} Transfer from Pascal VOC2007 \cite{pascal07} to TU Berlin \cite{eitz2012hdhso} (images-to-sketches transfer) has never been attempted yet in domain adaptation to our best knowledge. We utilize 50 and 3 source and target training samples per class, respectively, and the 14 classes that are common between the source and target datasets. We perform testing on the remaining target data. We report the mean average accuracy over 5 data splits.

\subsection{Experimental Setup}

In each stream, we employ the AlexNet architecture~\cite{krizhevsky_alexnet}  which was pre-trained on the ImageNet dataset~\cite{ILSVRC15} for the best results. At the training and testing time, we use the pipelines shown in Figures \ref{fig:cnn1} and \ref{fig:cnn2}, respectively. Where stated, we use the 16-layer VGG model~\cite{simonyan_vgg} per stream to quantify the impact of different CNN models on our algorithm. We set non-zero learning rates on the fully-connected and the last two convolutional layers of the two streams.

On the RGB-D-Caltech256 dataset, we use the RGB images from Caltech256 as the target domain. In contrast to~\cite{cvpr_rgbd-caltech,eccv_infbottleneck} which use both the RGB data and depth maps as a source, we adapt our source stream based on AlexNet to use only the depth data from the RGB-D dataset -- this helps us isolate performance of our algorithm in case of distinct heterogeneous domains. As these both domains are very different from each other, 
we apply two classifiers -- one per network stream (see the footnote\textsuperscript{\ref{foot:classifier}} in Section \ref{sec:approach}).

We evaluate Second- and/or Higher-order Transfer of Knowledge ({\em So-HoT}) approaches such as: unweighted and weighted second-order alignment losses ({\em So}) and ({\em So}+$\vzeta$), the third-order loss ({\em To}) and its weighted variant ({\em To+$\vzeta$}), and combined second- and third-  ({\em So}+{\em To}+$\vzeta$) as well as fourth-order ({\em So}+{\em To}+{\em Fo}+$\vzeta$) weighted alignment losses.
The model parameters were selected by cross-validation.

\begin{table*}[!t]
{
\centering
\begin{center}
\vspace{-0.4cm}
{
\setlength{\tabcolsep}{0.3em}
\centering
\begin{tabular}{c c|c c c c c c c c c c|c}
& & sp1    & sp2   & sp3  & sp4   & sp5   & sp6   & sp7   & sp8   & sp9   & sp10  & aver. acc.  \\ \hline
\multirow{6}{*}{\rotatebox[origin=c]{90}{\dsAW}}														& \kern-0.9em{\em S+T}								& 91.5 & 87.9 & 91.5 & 89.6 & 89.0 & 87.2 & 86.2 & 87.2 & 91.2 & 86.5 & 88.76$\pm$1.9 \\
																								& \kern-0.9em{\em So}								& 91.9 & 89.3 & 91.7 & 90.6 & 89.0 & 88.2 & 87.6 & 87.6 & 91.9 & 87.2 & 89.50$\pm$1.8\\\cline{2-13}
& \kern-0.3em{\em So}+$\vzeta$ & 92.4 & 89.9 & 90.5 & 92.2 & 88.9 & 88.2 & 90.0 & 89.5 & 91.3 & 89.6 & \textbf{90.24}$\pm$\textbf{1.3}\\
& \kern-0.3em{\em To}+$\vzeta$ & 92.5 & 90.3 & 91.8 & 91.9 & 89.0 & 89.9 & 89.3 & 89.6 & 91.6 & 89.6 & \textbf{90.55}$\pm$\textbf{1.1}\\
& \kern-1.2em{\em So+To}+$\vzeta$ & 92.6 & 90.5 & 92.0 & 92.0 & 89.2 & 90.0 & 89.6 & 89.9 & 91.8 & 89.9 & \textbf{90.75}$\pm$\textbf{1.2}\\
& \kern-1.2em{\em So+To+Fo}+$\vzeta$ &93.0 & 90.7 & 92.1 & 92.9 & 89.2 & 89.9 & 89.6 & 89.9 & 92.0 & 89.7 & \textbf{90.92}$\pm$\textbf{1.3}\\\hline
\multirow{3}{*}{\rotatebox[origin=c]{90}{\dsAD}} & \kern-0.9em{\em S+T} & 90.6 & 88.9 & 89.4 & 92.4 & 90.1 & 87.2 & 91.1 & 88.2 & 90.9 & 89.4 & 89.83$\pm$1.4\\
&\kern-0.9em\em{So} & 92.4 & 92.4 & 91.1 & 92.4 & 92.9 & 89.6 & 93.4 & 91.9 & 94.1 & 92.6 & \textbf{92.26}$\pm$\textbf{1.1}\\\cline{2-13}
&\kern-0.9em\em{So+To}+$\vzeta$ & 92.7 & 92.9 & 91.6 & 92.5 & 93.3 & 89.7 & 93.7 & 91.9 & 94.0 & 93.0 & \textbf{92.52}$\pm$\textbf{1.2}\\
& \kern-1.2em{\em So+To+Fo}+$\vzeta$ & 93.1 & 93.1 & 92.0 & 92.7 & 93.3 & 89.9 & 94.1 & 91.9 & 94.0  & 93.4 & \textbf{92.73}$\pm$\textbf{1.1}\\ 
\end{tabular}
\caption{The Office dataset on VGG streams. (Top) \dsAW~and (Bottom) \dsAD~domain shifts are evaluated on second-order ({\em So}), second- ({\em So}+$\vzeta$) and third-order+weights ({\em To}+$\vzeta$), second- and third- ({\em So+To}+$\vzeta$) and fourth-order ({\em So+To+Fo}+$\vzeta$) alignment with weight learning. Our baseline fine-tuning on the combined source and target domains ({\em S+T}) is also evaluated for comparison.
\label{tab:so_weighted}}
}
\end{center}
}
\vspace{-0.6cm}
\end{table*}

\subsection{Comparison to the State of the Art}

We apply our algorithm on the Office dataset. Table~\ref{tab:office} presents results for the six domain shifts. Our second-order alignment loss ({\em So}) is compared against the baseline ({\em S+T}) for which the source and target training samples were used together to fine-tune a standard CNN network. As can be seen, our method outperforms such a baseline as well as recent approaches such as {\em Domain Confusion with Soft Labels} and fine-tuning on the source or target data, respectively. 

\vspace{0.05cm}
\noindent{\textbf{Performance on the VGG architecture.}}
To evaluate effectiveness of our algorithm on other powerful networks, we follow the same pipeline as in Figure \ref{fig:cnn_all}, except that we employ the pre-trained VGG~\cite{simonyan_vgg} in place of AlexNet~\cite{krizhevsky_alexnet}. As VGG utilizes more parameters than AlexNet, we demonstrate in Table \ref{tab:so_weighted} that applying our second-order alignment loss ({\em So}) on \dsAW~and \dsAD~improves performance compared to the baselines ({\em S+T}) by 0.74\% and 2.43\%. 
Without resorting to data augmentations, we outperform \eg~a multi-scale multi-patch CNN approach~\cite{kuzborskij_cvpr16} by 0.6\% on \dsAW. 
%

\vspace{0.05cm}
\noindent{\textbf{Weighted vs. Unweighted Alignment.}}
\begin{figure}[b]
\vspace{-0.4cm}
\centering
\parbox{0.55\textwidth}{
%
\vspace{-0.2cm}
\centering
\begin{subfigure}[b]{0.48\linewidth}
\centering\includegraphics[trim=0 0 0 0, clip=true, height=3cm]{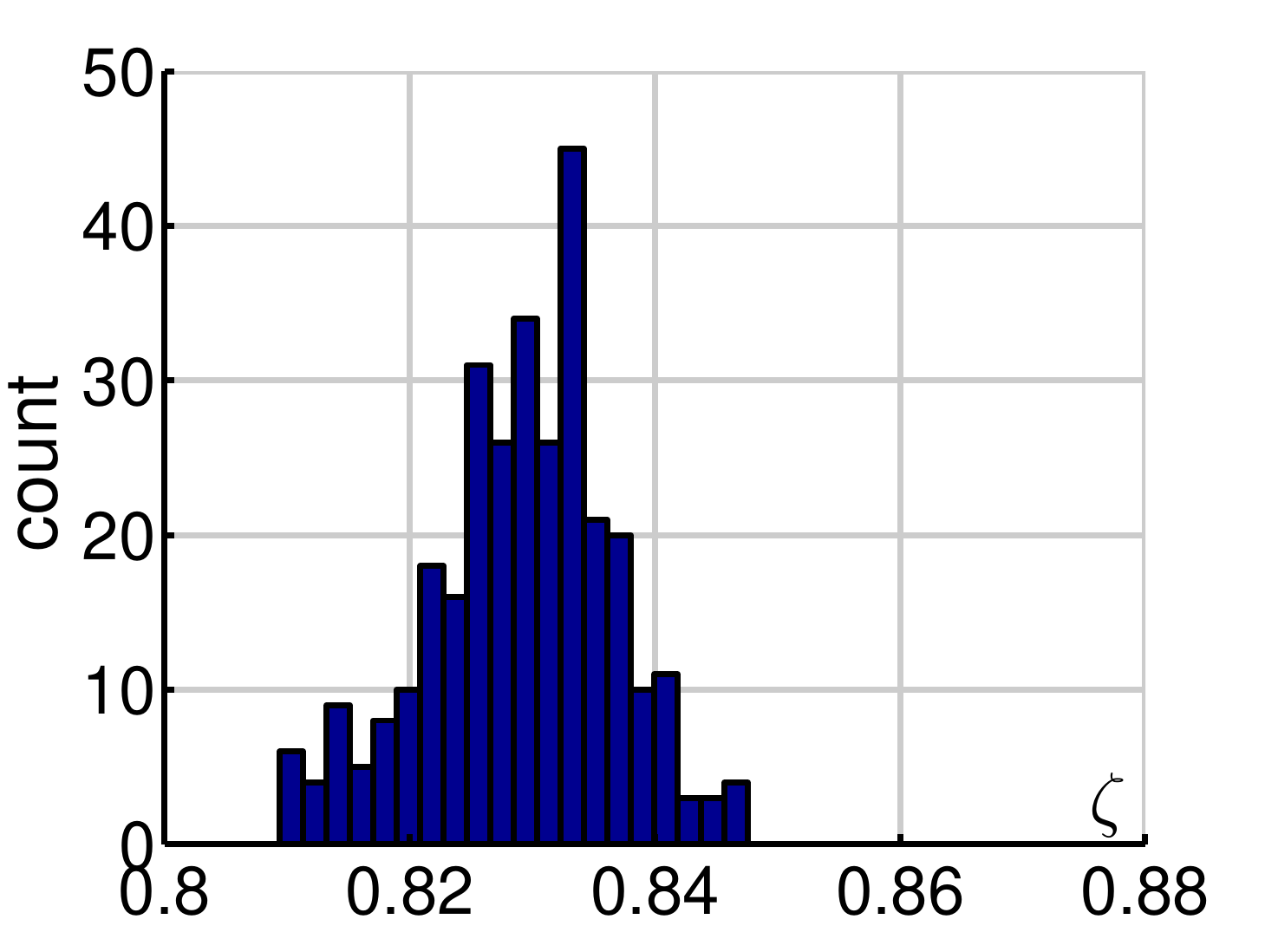}
\caption{\label{fig:hist1}}
\end{subfigure}
%
\begin{subfigure}[b]{0.48\linewidth}
\centering\includegraphics[trim=0 0 0 0, clip=true, height=3cm]{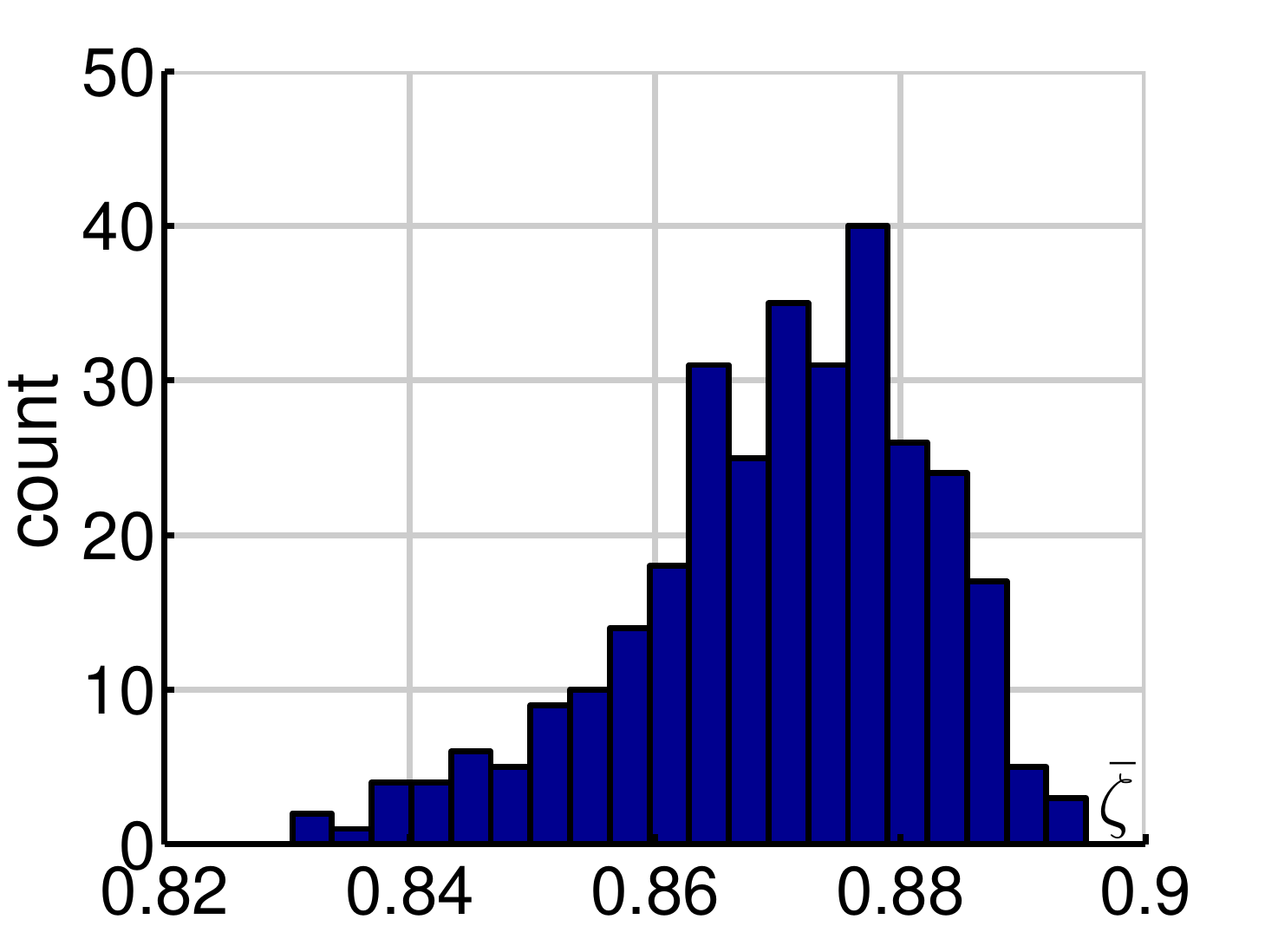}
\caption{\label{fig:hist2}}
\end{subfigure}
\vspace{-0.2cm}
\caption{Histograms of the $\vzeta$ and $\vec{\overbartwo{\zeta}}$ weights in plots \ref{fig:hist1} and \ref{fig:hist2}, learned on~\dsAW, show the level of alignment of the scatter matrices and their means according to the loss function in \eqref{eq:main_obj2}.}
\vspace{-0.3cm}
\label{fig:hists}
}
\hspace{0.2cm}
\parbox{0.42\textwidth}{
\centering
\includegraphics[trim=0 0 0 0, clip=true, height=2.7cm]{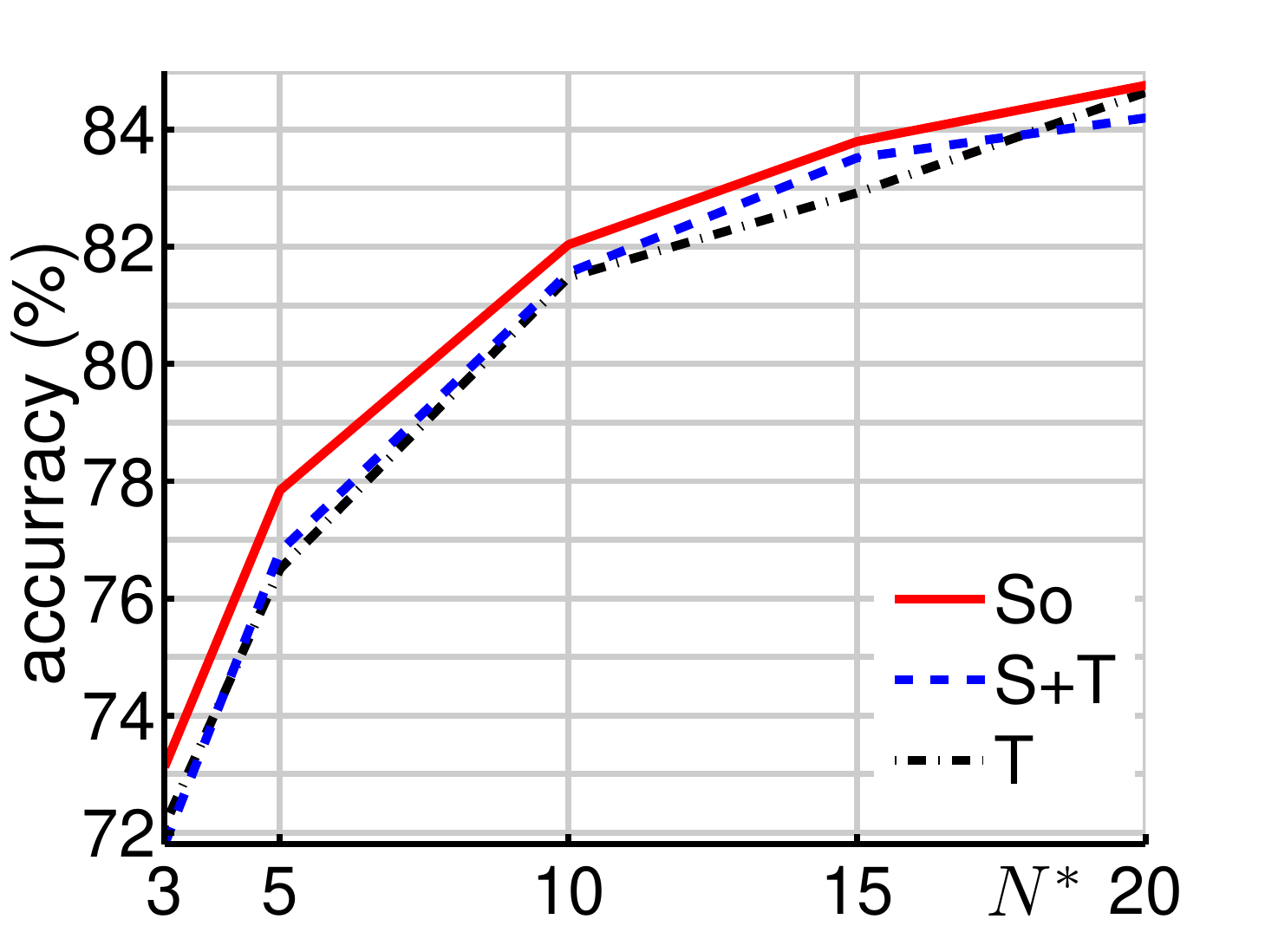}
%
\caption{Our second-order algorithm ({\em So}) vs. the baseline fine-tuning on: i) the combined source and target domains ({\em S+T}) and ii) the target domain only ({\em T}). We use RGB-D-Caltech256 (heterogeneous setup). $N^*\!$ is the number of  target tr. samples per class.}
\vspace{-0.0cm}
\label{fig:rgbd}
}
\end{figure}
%
%
%
%
%
%
In this experiment, we demonstrate the benefit of using the weighted alignment of the scatter matrices and their means on the \dsAW~and \dsAD~domain shits. Table \ref{tab:so_weighted} shows that our weighted second- ({\em So}+$\vzeta$) and third-order ({\em To}+$\vzeta$) alignment losses, introduced in Eq. \eqref{eq:main_obj2}, improve over our unweighted second-order alignment loss ({\em So}) from Eq. \eqref{eq:main_obj1} by 0.74\% and 1.05\% on \dsAW, respectively. 
Learning $\vzeta$ and $\vec{\overbartwo{\zeta}}$ can be implemented at no visible increase in computations. 

In Figure \ref{fig:hists}, we show histograms of the $\vzeta$ and $\vec{\overbartwo{\zeta}}$ weights from ({\em So}+$\vzeta$) over the 31 classes and the 10 splits. The histograms reveal that the levels of alignment of the scatter matrices and their means vary according to the Beta distributions. The means of these distributions are slightly below the desired mean value of one which indicates that,  in this experiment, $\sigma_1$ and $\sigma_2$ from Eq. \eqref{eq:main_obj2} were initialized with values larger than needed. Also, their optimal values might vary over time -- learning weights compensates for this.
 
\vspace{0.05cm}
\noindent{\textbf{Alignment of combined Second-, Third- and Fourth-order Scatter Tensors.}}
Our kernelized loss in Eq. \eqref{prop:kernelised2} admits alignment between second- and/or higher-order scatter tensors which, beyond the scale/shear and orientation, capture higher-order statistical moments. In Table \ref{tab:so_weighted}, we evaluate third-order weighted alignment loss ({\em To}+$\vzeta$), as well as combined second-, third- ({\em So}+{\em To}+$\vzeta$) and fourth-order ({\em So}+{\em To}+{\em Fo}+$\vzeta$) weighted alignments. As the order increases, the performance improves. For ({\em So}+{\em To}+{\em Fo}+$\vzeta$), we outperform ({\em So}) by 1.42\% and 2.9\% on \dsAW~and \dsAD.

\vspace{0.05cm}
\noindent{\textbf{Heterogeneous setting on RGB-D-Caltech256.}}
\noindent{\textbf{Heterogeneous setting on RGB-D-Caltech256.}} In this experiment, we verify the behavior of our second-order alignment loss ({\em So}) w.r.t. the varying number of target training samples $N^*\!\!$. Figure \ref{fig:rgbd} shows that the largest improvement of 1.24\% and 1.04\% over the baseline ({\em S+T}) is obtained for a small number $N^*\!\!=\!3$ and $N^*\!\!=\!5$, respectively. 
%
%
%
As $N^*\!$ increases, the improvement over baselines becomes smaller. Such a trend is consistent with other works on domain adaptation \cite{tommasi_cvpr10}. In some cases, the baseline ({\em S+T}) performs worse than the fine-tuning on target only ({\em T}) which is known as so-called {\em negative transfer} \cite{tommasi_cvpr10}. For all $3\!\leq\!N^*\!\!\leq\!20$, our ({\em So}) outperforms baselines ({\em S+T}) and ({\em T}) which demonstrates robustness of our approach.

\vspace{0.05cm}
\noindent{\textbf{Heterogeneous setting on Pascal VOC2007-TU Berlin.}} %
Table \ref{tab:voc-sketch} shows results on transfer from Pascal VOC2007 \cite{pascal07} to TU Berlin \cite{eitz2012hdhso} (images-to-sketches transfer). These dataset have never been used together in domain adaptation. We utilize 50 and 3 source and target training samples per class, respectively, and the 14 classes that are common between the source and target datasets. We use AlexNet streams in this experiment. As demonstrated in the table, our second-order approach ({\em So}) and the baselines ({\em S+T}) and ({\em T}) yield \textbf{63.4}, 62.66 and 62.46\%, respectively. 
 
\vspace{0.05cm}
\noindent{\textbf{Comparisons to CORAL on the Office dataset.}} %
To compare ({\em So}) to CORAL \cite{frustrating_domain_return}, we modified our code to align second-order marginal statistics ({\em M}) in the supervised setting. 
For AlexNet and \dsAW~, ({\em M}) scores $82.6\%$ vs. baseline ({\em S+T}) of $82.4\%$ but is below $84.5\%$ from our ({\em So}). 
On \dsDW~, ({\em M}) gave $94.6\%$ vs. $95.5\%$ from our ({\em So}). 
On \dsWD~, ({\em M}) gave $95.9\%$ vs. $97.5\%$ from our ({\em So}). 

\begin{table}[t]
\begin{center}
{
\setlength{\tabcolsep}{0.3em}
\centering
\begin{tabular}{c|c c c c c|c}
					& sp1   & sp2   & sp3   & sp4   & sp5 & average acc.  \\ \hline
{\em S+T}	& 58.86 & 63.43 & 63.14 & 59.14 & 68.71 & 62.66\\ 
{\em T}	& 59.86 & 63.43 & 64.14 & 57.86 & 67.0  & 62.46\\                     
{\em So}	& 60.57 & 63.28 & 64.28 & 59.14 & 69.71 & \textbf{63.40} \\ \hline
\end{tabular}
}
\end{center}
\vspace{-0.2cm}
\caption{Pascal VOC2007-TU Berlin dataset. We use 5 splits and report accuracies on our method ({\em So}) vs. baselines ({\em S+T}) and ({\em T}).}
\label{tab:voc-sketch}
\vspace{-0.4cm}
\end{table}

\section{Conclusions}
\begin{figure}[b]
\vspace{-0.4cm}
\centering
\hspace{0.1cm}
\begin{subfigure}[b]{0.3\linewidth}
\centering\includegraphics[trim=0 0 0 0, clip=true, height=3cm]{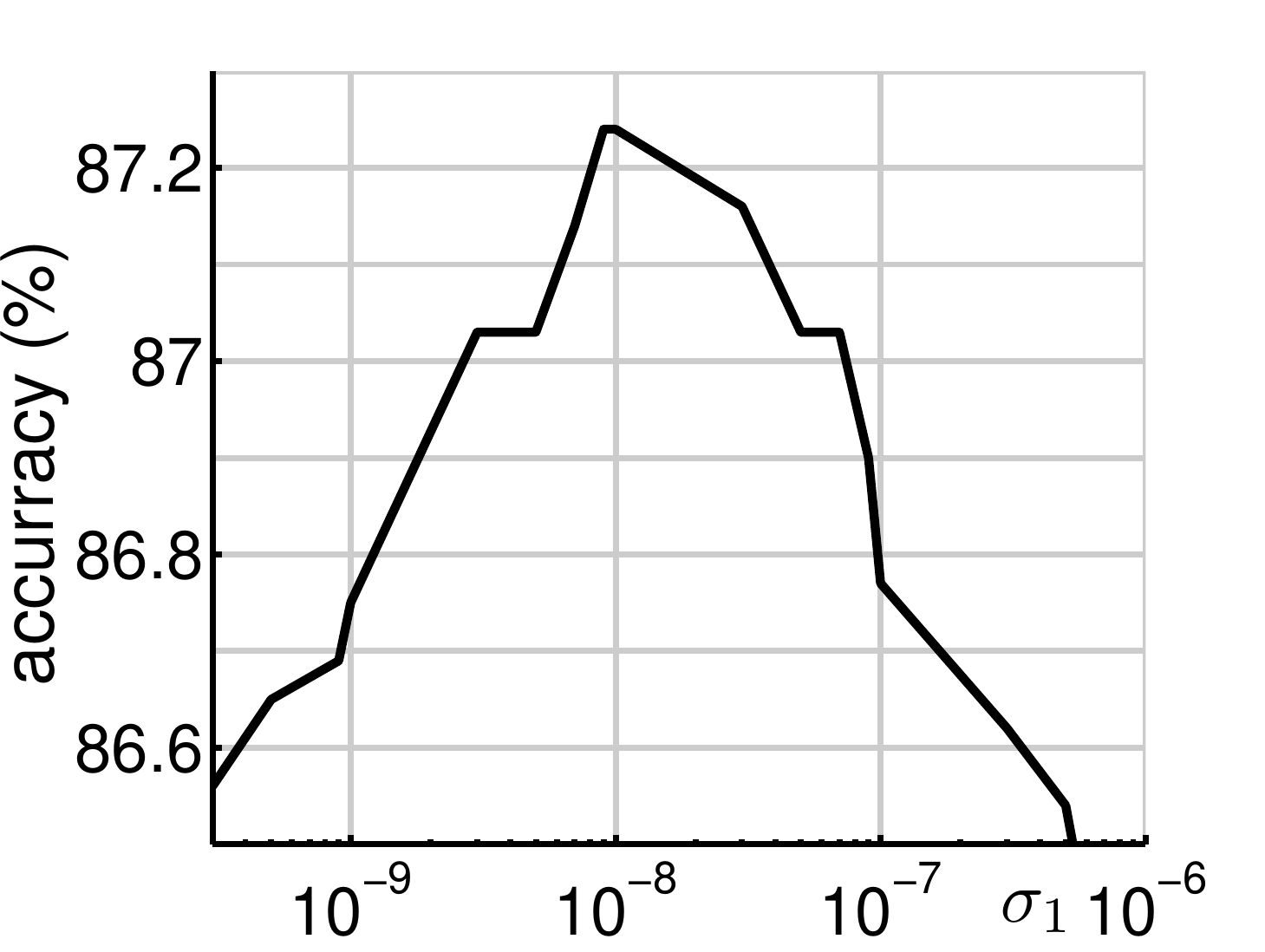}
\caption{\label{fig:sigma1}}
\end{subfigure}
\begin{subfigure}[b]{0.3\linewidth}
\centering\includegraphics[trim=0 0 0 0, clip=true, height=3cm]{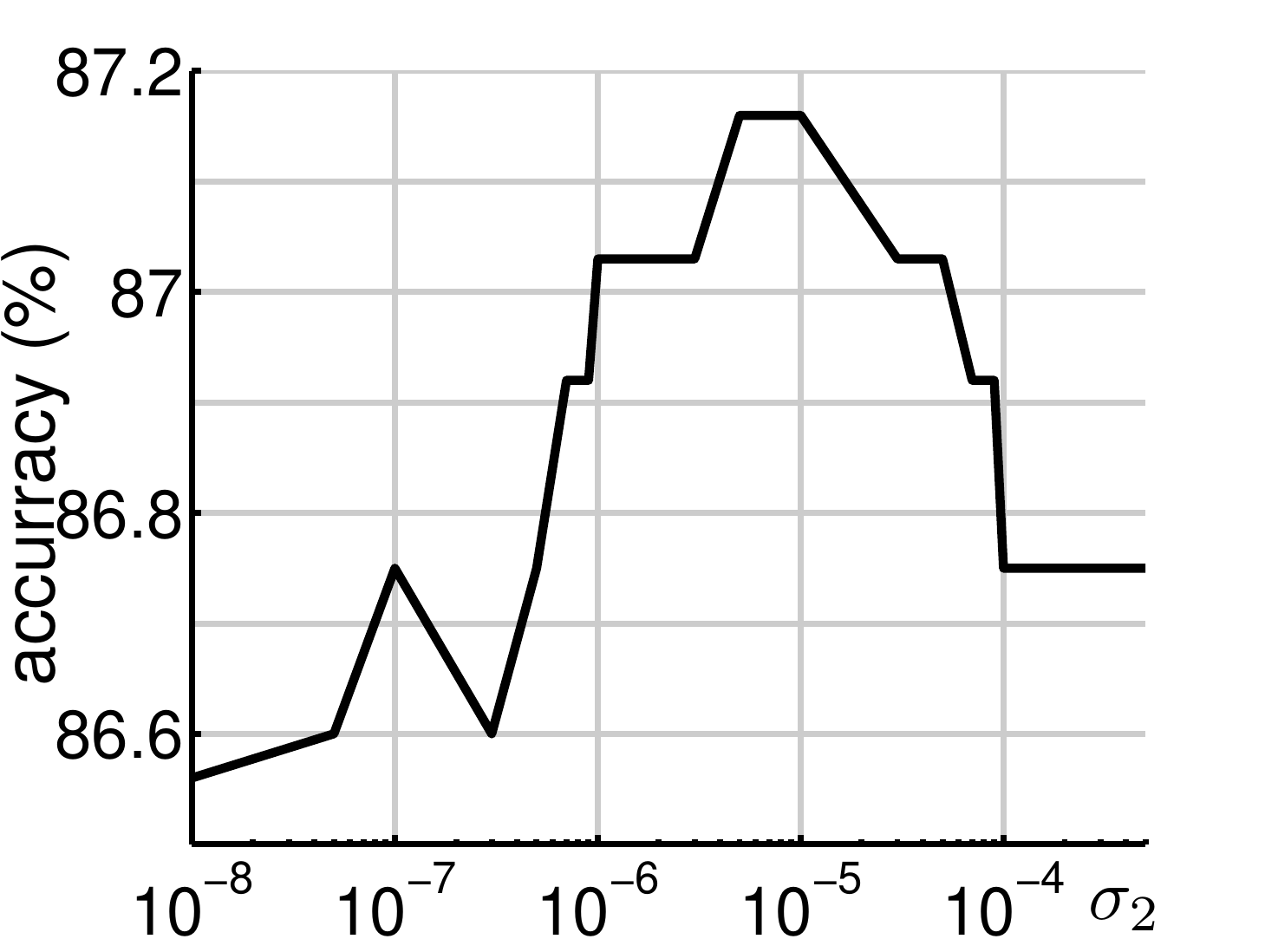}
\caption{\label{fig:sigma2}}
\end{subfigure}
\vspace{-0.2cm}
\caption{Performance of our second-order alignment loss ({\em So}) w.r.t. parameters $\sigma_1$ (\ref{fig:sigma1}) and $\sigma_2$ (\ref{fig:sigma2}) on the~\dsAW~domain shift (split {\em sp1} is used). Note the logarithmic scale.}
\vspace{-0.1cm}
\label{fig:sigmas}
\end{figure}

We have presented an approach to domain adaptation by partial alignment of the within-class scatters to discover the commonality. The state-of-the-art results we obtain suggest that our simple strategy is effective despite challenges of domain adaptation. Moreover, the presented weighted approach and kernelized alignment loss improve the results and computational efficiency. Our method can be easily extended to multiple domains and other network architectures. Our supplementary material presents additional results.

\vspace{0.05cm}
\noindent{\textbf{Acknowledgements.}} We thank Dr. Hongping Cai and Dr. Peter Hall for our early discussions on the alignment loss. We thank NVIDIA for the donation of GPUs.

%% file: appendix.tex
\vspace{-0.4cm}
\begin{appendices}
\vspace{-0.1cm}

\section{Sensitivity to parameters $\sigma_1$ and $\sigma_2$}
Below we evaluate sensitivity of our second-order alignment loss ({\em So}) to the parameter $\sigma_1$ which determines the level of scatter alignment (rotation and scale/shear) as well as $\sigma_2$ which determines the level of alignment of means.

Figure \ref{fig:sigmas} shows that, as $\sigma_1\!\shortrightarrow\!0$ and $\sigma_2\!\shortrightarrow\!0$, our algorithm converges to the baseline fine-tuning on the combined source and target domains ({\em S+T}) which yielded 85.9\% accuracy for split {\em sp1}. Moreover, the overall performance is stable \ie, within $\pm$0.2\% accuracy, for a large range of 
values \eg, $5e\!-\!9\!\leq\!\sigma_1\!\leq\!5e\!-\!8$ and $1e\!-\!6\!\leq\!\sigma_2\!\leq\!5e\!-\!5$.

\section{Derivatives of the alignment loss $g$ w.r.t. the feature vectors}
\label{app:der_sec_ord}

Suppose $\mPhi\!=\![\vphi_1,\cdots,\vphi_N]$ and $\mPhi^{*}\!\!=\![\vphi^*_!\!,\cdots,\vphi^*_N\!]$ are some feature vectors of quantity $N$ and $N^{*\!}$, respectively, which are used to evaluate $\cov$ and $\cov^{*}\!$. 
For $r\!=\!2$, we have to first compute the derivative of the covariance matrix $\cov$ w.r.t. $\vphi_{m'n'}$. To do so, we proceed by computing derivatives of: i) the autocorrelation matrix in \eqref{eq:der_auto} and ii) the outer product of means $\vmu$ in \eqref{eq:der_meanout} and \eqref{eq:der_meanout2}:
\begin{align}
& \!\!\!\frac{\partial\sum_n\!\vphi_n\vphi_n^T}{\partial \phi_{m'n'}}\!=\!\vj_{m'}\vphi_{n'}^T\!+\!\vphi_{n'}\vj_{m'}^T,\label{eq:der_auto}\\
& \!\!\!\frac{\partial\vmu\vmu^T}{\partial \mu_{m'}}\!=\!\vj_{m'}\vmu^T\!+\!\vmu\vj_{m'}^T,\label{eq:der_meanout}\\
& \!\!\!\frac{\partial\vmu\vmu^T}{\partial\phi_{m'n'}}\!=\!\sum_m\!\frac{\partial\vmu\vmu^T}{\partial \mu_m}\frac{\partial \mu_m}{\partial\phi_{m'n'}}\!=\!\frac{1}{N}\!\left(\vj_{m'}\vmu^T\!+\!\vmu\vj_{m'}^T\right),\label{eq:der_meanout2}
\end{align}
where $\vj_{m'}$ is a vector of zero entries except for position $m'$ which is equal one. 
Putting together \eqref{eq:der_auto}, \eqref{eq:der_meanout} and \eqref{eq:der_meanout2} yields the derivative of $\cov$ w.r.t. $\vphi_{m'n'}$:
\begin{align}
& \frac{\partial\left(\frac{1}{N}\!\sum_n\!\vphi_n\vphi_n^T\right)\!-\!\vmu\vmu^T}{\partial \phi_{m'n'}}\!=\!\frac{1}{N}\!\left(\vj_{m'}\left(\vphi_{n'}\!-\!\vmu\right)^T\!+\!\left(\vphi_{n'}\!-\!\vmu\right)\vj_{m'}^T\right).\label{eq:der_cov}
\end{align}

\vspace{0.05cm}
\noindent{\textbf{The derivatives}} of $||\cov\!-\!\cov^{*}||_F^2$ w.r.t. covariance $\cov$ as well as $\vphi_{m'n'}$ and $\vphi^*_{m'n'}$ are provided below:
\begin{align}
& \frac{\partial ||\cov\!-\!\cov^{*}||_F^2}{\partial\cov}\!=\!2\left(\cov\!-\!\cov^{*}\right)\\
& \frac{\partial ||\cov\!-\!\cov^{*}||_F^2}{\partial\phi_{m'n'}}\!=\!\sum_{m,n}\!\frac{\partial ||\cov\!-\!\cov^{*}||_F^2}{\partial\Sigma_{mn}}\left(\frac{\partial\cov}{\partial \phi_{m'n'}}\right)_{mn}\nonumber\\
&\;=\!\frac{2}{N}\!\sum_{m,n}\!\left(\!\cov\!-\!\cov^{*}\!\right)_{mn}\left(\vj_{m'}\left(\vphi_{n'}\!-\!\vmu\right)^T\!\!+\!\left(\vphi_{n'}\!-\!\vmu\right)\vj_{m'}^T\right)_{mn}\nonumber\\
&\;=\!\frac{4}{N}\!\left(\!\cov_{m',:}\!-\!\cov_{m',:}^{*}\right)\!\left(\vphi_{n'}\!-\!\vmu\right).
\end{align}
\noindent{\textbf{The derivatives}} of $||\cov\!-\!\cov^{*}||_F^2$ w.r.t. $\mPhi$ and $\mPhi^{*}\!$ are:
\begin{align}
& \frac{\partial ||\cov\!-\!\cov^{*}||_F^2}{\partial\mPhi}\!=\!\frac{4}{N}\!\left(\!\cov\!-\!\cov^{*}\!\right)\!\left(\mPhi\!-\!\vmu\vOnes^T\right),\\
& \frac{\partial ||\cov\!-\!\cov^{*}||_F^2}{\partial\mPhi^{*}}\!=\!-\frac{4}{N^*}\!\left(\!\cov\!-\!\cov^{*}\!\right)\!\left(\mPhi^{*}\!-\!\vmu^{*}\vOnes^T\right).
\end{align}

\noindent{\textbf{The derivatives}} of $||\vmu\!-\!\vmu^{*}||_2^2$ w.r.t. $\vmu$, $\vphi_n$ and $\vphi^{*}_{n'}$ are:
\begin{align}
& \!\!\!\!\!\frac{\partial ||\vmu\!-\!\vmu^{*}||_2^2}{\partial\vmu}\!\!=\!2\left(\vmu\!-\!\vmu^{*}\right),\\
& \!\!\!\!\!\frac{\partial ||\vmu\!-\!\vmu^{*}||_2^2}{\partial\vphi_{n'}}\!\!=\!\!\frac{2\left(\vmu\!-\!\vmu^{*}\right)}{N},\;\frac{\partial ||\vmu\!-\!\vmu^{*}||_2^2}{\partial\vphi_{n'}^{*}}\!\!=\!\!\frac{2\left(\vmu\!-\!\vmu^{*}\right)}{N^*}.\!\!
\end{align}

\section{Kernelized derivative of the Frobenius norm between tensors w.r.t. the feature vectors}
\label{app:der_ker_ord}

Suppose that some feature vectors $\mPhi\!=\![\vphi_1,\cdots,\vphi_N]$ and $\mPhi^{*\!}\!=\![\vphi^*_!\!,\cdots,\vphi^{*\!}_N]$  are given in quantities $N$ and $N^{*\!}$ 
and that the Frobenius norm between tensors $\tX^{(r)}$ and $\tY^{(r)}$ of order $r\!\geq\!1$ build from $\mPhi$ and $\mPhi^{*\!}$ is being evaluated. Then, the derivative of Equation \eqref{prop:kernelised2} w.r.t. feature vector $\vphi_{n^\ddag\!}$ becomes:
\begin{align}
&\frac{\partial ||\tX^{(r)}\!-\!\tX^{*(r)}||_F^2}{\partial\vphi_{n^\ddag\!}}\!=\!\frac{1}{N^2}r\sum\limits_{n=1}^N\sum\limits_{n'=1}^{N}{K_{nn'}^{r-1}}\frac{\partial K_{nn'}}{\partial \vphi_{n^\ddag\!}}\nonumber\\
&\qquad\qquad\qquad\qquad\!\!-\frac{2}{NN^{*\!}}r\sum\limits_{n=1}^N\sum\limits_{n'=1}^{N^*}{\Kbb_{nn'}^{r-1}}\frac{\partial \Kbb_{nn'}}{\partial \vphi_{n^\ddag\!}},\label{eq:third_ord_inner_df1}
\end{align}
where
\begin{align}
&\frac{\partial K_{nn'}}{\partial \vphi_{n^\ddag\!}}\!=\!\frac{\partial\left<\vphi_n,\vphi_{n'}\right>}{\partial\vphi_{n^\ddag\!}}\!-\!\frac{\partial\left<\vmu,\vphi_{n'}\right>}{\partial\vphi_{n^\ddag\!}}\!-\!\frac{\partial\left<\vphi_n,\vmu\right>}{\partial\vphi_{n^\ddag\!}}\!+\!\frac{\partial\left<\vmu,\vmu\right>}{\partial\vphi_{n^\ddag\!}}\nonumber\\
&=\!\left \{
  \begin{aligned}
    &\vphi_{n'}&& &&\frac{\vphi_{n'}}{N}&& &(\vmu\!+\!\frac{\vphi_{n}}{N})&& &\!\!\!\!, &&\!\!\!\!\text{if}\,n\!=\!{n^\ddag\!}, {n'}\!\!\neq\!{n^\ddag\!} \\
    &\vphi_n&&\raisebox{-2.5\height}{$\!\!\!\!-\!\!\!\!$}&&(\vmu\!+\!\frac{\vphi_{n'}}{N})&&\raisebox{-2.5\height}{$\!\!\!\!-\!\!\!\!$}&\frac{\vphi_{n}}{N}&&\raisebox{-1.7\height}{$\!\!\!\!+\frac{2}{N}\vmu$}&\!\!\!\!, &&\!\!\!\!\text{if}\,n\!\neq\!{n^\ddag\!}, {n'}\!\!=\!{n^\ddag\!} \\[-0.2cm]
    &2\vphi_{n}&& &&(\vmu\!+\!\frac{\vphi_{n'}}{N})&& &(\vmu\!+\!\frac{\vphi_{n}}{N})&& &\!\!\!\!, &&\!\!\!\!\text{if}\,n\!=\!{n^\ddag\!}, {n'}\!\!=\!{n^\ddag\!} \\
    &0&& &&\frac{\vphi_{n'}}{N}&& &\frac{\vphi_{n}}{N}&& &\!\!\!\!, &&\!\!\!\!\text{if}\,n\!\neq\!{n^\ddag\!}, {n'}\!\!\neq\!{n^\ddag\!},\label{eq:third_ord_inner_df2}
  \end{aligned} \right.
\end{align}
\vspace{-0.5cm}
%
%
%
%
\begin{align}
&\frac{\partial \Kbb_{nn'}}{\partial \vphi_{n^\ddag\!}}\!=\!\frac{\partial\left<\vphi_n,\vphi^*_{n'}\right>}{\partial\vphi_{n^\ddag\!}}\!-\!\frac{\partial\left<\vmu,\vphi^*_{n'}\right>}{\partial\vphi_{n^\ddag\!}}\!-\!\frac{\partial\left<\vphi_n,\vmu^*\right>}{\partial\vphi_{n^\ddag\!}}\!+\!\frac{\partial\left<\vmu,\vmu^*\right>}{\partial\vphi_{n^\ddag\!}}\nonumber\\
&=\!\left \{
  \begin{aligned}
&\vphi^*_{n'}&&                                                         && &\vmu^*&& &, &&\!\!\!\!\text{if}\,n\!=\!{n^\ddag\!}, {n'}\!\!\neq\!{n^\ddag\!} \\
&0           &&\raisebox{-1.65\height}{$\!\!\!\!-\frac{1}{N}\vphi^*_{n'}$}&& \raisebox{-2.4\height}{$-$} & 0 && \raisebox{-1.7\height}{$\!\!\!\!+\frac{1}{N}\vmu^*$}&, &&\!\!\!\!\text{if}\,n\!\neq\!{n^\ddag\!}, {n'}\!\!=\!{n^\ddag\!} \\[-0.2cm]
&\vphi^*_{n}&& && &\vmu^*&& &, &&\!\!\!\!\text{if}\,n\!=\!{n^\ddag\!}, {n'}\!\!=\!{n^\ddag\!} \\
&0          && && &0&& &, &&\!\!\!\!\text{if}\,n\!\neq\!{n^\ddag\!}, {n'}\!\!\neq\!{n^\ddag\!}.
\end{aligned} \right.\label{eq:third_ord_inner_df3}\nonumber\\[-0.5cm]
\end{align}
Putting together Equations \eqref{eq:third_ord_inner_df1}, \eqref{eq:third_ord_inner_df2} and \eqref{eq:third_ord_inner_df3} and setting $q\!=\!r\!-\!1$ yields the derivatives w.r.t. matrices 
$\mPhi$ and $\mPhi^{*\!}$:
\begin{align}
&\frac{\partial ||\tX^{(r)}\!-\!\tX^{*(r)}||_F^2}{\partial\mPhi}\!=\frac{2}{N^2}r\mPhi\left(\mKro^T\!\!\!-\!\frac{1}{N}(\vOnes^T\mKro)^T\vOnes^T\right)\nonumber\\
&+\!\frac{2r\vmu}{N^2}\!\left(\frac{1}{N}\vOnes^T\mKro\vOnes-\vOnes^T\mKro\right)
\!-\!\frac{2r\mPhi^{*\!}}{NN^*}\!\left(\mKbbro^T\!\!\!-\!\frac{1}{N}(\mKbbro^T\vOnes)\vOnes^T\right)\nonumber\\
&+\!\frac{2}{NN^*}r\vmu^*\left(\vOnes^T\mKbbro^T\!\!\!-\!\frac{1}{N}\vOnes^T\mKbbro^T\vOnes\right)\!\!
\end{align}
and
\begin{align}
&\frac{\partial ||\tX^{(r)}\!-\!\tX^{*(r)}||_F^2}{\partial\mPhi^*}\!=\frac{2}{{N^*}^2}r\mPhi^*\left(\mKbro^T\!\!\!-\!\frac{1}{N^*}(\vOnes^T\mKbro)^T\vOnes^T\right)\nonumber\\
&+\!\frac{2r\vmu^{*\!}}{{N^*}^2}\!\left(\frac{1}{N^*}\vOnes^T\mKbro\vOnes-\vOnes^T\mKbro\right)
\!-\!\frac{2r\mPhi}{NN^*}\!\left(\mKbbro\!-\!\frac{1}{N^*}(\mKbbro\vOnes)\vOnes^T\right)\nonumber\\
&+\!\frac{2}{NN^*}r\vmu\left(\vOnes^T\mKbbro\!-\!\frac{1}{N^*}\vOnes^T\mKbbro\vOnes\right).
\end{align}

\end{appendices}